\documentclass[accepted]{uai2025} 
                        
\usepackage[american]{babel}
\usepackage{amsthm,amsmath}
\usepackage{amsfonts}
\usepackage[ruled,vlined]{algorithm2e}
\theoremstyle{definition}
\newtheorem{defn}{Definition}
\theoremstyle{plain}
\newtheorem{thm}{Theorem}
\newtheorem{lem}[thm]{Lemma}

\usepackage[utf8]{inputenc}
\usepackage[T1]{fontenc}
\usepackage{optidef}
\usepackage{hyperref}
\usepackage{caption}
\usepackage{subcaption}
\usepackage{comment}
\SetCommentSty{mycommfont}
\usepackage{multirow}
\hypersetup{
    colorlinks=true,
    linkcolor=blue,
    filecolor=magenta,      
    urlcolor=black,
}

\newcommand{\argmin}{\mathop{\rm arg~min}\limits}
\DeclareUnicodeCharacter{3000}{}

\usepackage{natbib} 
\usepackage{mathtools} 
\usepackage{booktabs} 
\usepackage{tikz} 

\title{Multiple Wasserstein Gradient Descent Algorithm for Multi-Objective Distributional Optimization}

\author[1]{\href{mailto:<haidnguyen0909@gmail.com>?Subject=Your UAI 2025 paper}{Dai Hai Nguyen}{}}
\author[2]{\href{mailto:<mami@kuicr.kyoto-u.ac.jp>?Subject=Your UAI 2025 paper}{Hiroshi Mamitsuka}{}}
\author[1]{\href{mailto:<atsu@ist.hokudai.ac.jp>?Subject=Your UAI 2025 paper}{Atsuyoshi Nakamura}{}}
\affil[1]{%
    Hokkaido University, Hokkaido, Japan
}
\affil[2]{%
    Kyoto University, Kyoto, Japan
}
  
\begin{document}
\maketitle
\begin{abstract}
We address the optimization problem of simultaneously minimizing multiple objective functionals over a family of probability distributions. This type of Multi-Objective Distributional Optimization commonly arises in machine learning and statistics, with applications in areas such as multiple target sampling, multi-task learning, and multi-objective generative modeling. To solve this problem, we propose an iterative particle-based algorithm, which we call Multiple Wasserstein Gradient Descent (MWGraD), which  constructs a flow of intermediate empirical distributions, each being represented by a set of particles, which gradually minimize the multiple objective functionals simultaneously.
Specifically, MWGraD consists of two key steps at each iteration. First, it estimates the Wasserstein gradient for each objective functional based on the current particles. Then, it aggregates these gradients into a single Wasserstein gradient using dynamically adjusted weights and updates the particles accordingly. In addition, we provide theoretical analysis and present experimental results on both synthetic and real-world datasets, demonstrating the effectiveness of MWGraD.
\end{abstract}

\section{Introduction}\label{sec:intro}
Many problems in machine learning and computational statistics turn into distributional optimization, where the goal is to minimize a functional $F:\mathcal{P}_{2}(\mathcal{X})\rightarrow \mathbb{R}$ over the set of probability distributions: $\min_{q\in \mathcal{P}_{2}(\mathcal{X})}F(q)$, where $\mathcal{P}_{2}(\mathcal{X})$ denotes the set of probability distributions defined on the domain $\mathcal{X}$ ($\subseteq \mathbb{R}^{d}$) with finite second-order moment. This formulation arises in a variety of well-known problems, including Bayesian inference (e.g., variational autoencoder \citep{kingma2014stochastic}) and synthetic sample generation (e.g. generative adversarial networks \citep{goodfellow2020generative}). These models aim to approximate a target distribution $\pi$ by generating samples (or also called particles) in a way that minimizes the dissimilarity $D(q,\pi)$ between the empirical probability distribution $q$ derived from the particles and the target distribution $\pi$. Common dissimilarity measures include Kullback-Leiber (KL) divergence, Jensen-Shanon (JS) divergence, and Wasserstein distance in the optimal transport \citep{villani2021topics}.  By setting $F(q)=D(q, \pi)$, the task can be framed as a distributional optimization problem. This problem can be solved using iterative algorithms, such as Wasserstein gradient descent \citep{zhang2016first}, which proceeds in two main steps at each iteration: (1) estimating the Wasserstein gradient of $F$ with respect to the current distribution, and (2) applying the exponential mapping to update the distribution on $\mathcal{P}_{2}(\mathcal{X})$. However, step (1) can be non-trivial when the target distribution $\pi$ is represented by a set of samples, making it challenging to estimate the Wasserstein gradient of $F$. To tackle this problem, Variational Transport (VT) \citep{liu2021infinite} is proposed. The key idea of VT is to assume that F has a variational form and reformulates $F$ as a variational maximization problem. Solving this problem allows to approximate the Wasserstein gradient of $F$ through samples from $\pi$, with resulting solution specifying a direction to update each particle. This can be viewed as a forward discretization of the Wasserstein gradient flow \citep{santambrogio2015optimal}. Inspired by VT, several other methods have been introduced to address variants of distributional optimization problems, including MirrorVT \citep{nguyen2023mirror} and MYVT \citep{nguyen2024moreau}.

On the other side, multi-objective optimization (MOO) \citep{deb2016multi} optimizes multiple objective functions simultaneously, and can be formulated as
\begin{equation}  
\label{definition:MOO}
\min_{\textbf{x}\in\mathcal{X}} \textbf{f}(\textbf{x})= \min_{\textbf{x}\in\mathcal{X}} f_{1}(\textbf{x}), f_{2}(\textbf{x}),..., f_{K}(\textbf{x}),
\end{equation}
where $K\geq 2$ represents the number of objectives, $f_{k}(\textbf{x})$ is the $k$-th objective function, and $\mathcal{X}\subseteq \mathbb{R}^{d}$ is the feasible set of $d$-dimensional vectors. 
Different from single-objective optimization, in MOO, it could have two vectors where one performs better for task $i$ and the other performs better for task $j\neq i$. Therefore, Pareto optimality is defined to deal with such an incomparable case. In particular, for two solutions $\textbf{x},\textbf{y}\in \mathcal{X}$, we say that $\textbf{y}$ dominates $\textbf{x}$ if $f_k(\textbf{x})\leq f_k(\textbf{y})$ for all $k\in [K]$, where $[K]$ denotes $\left\{ 1,2,3...,K\right\}$, and $\textbf{f}(\textbf{x})\neq \textbf{f}(\textbf{y})$. A solution $\textbf{x}$ is Pareto optimality if no other solution in $\mathcal{X}$ dominates it. MOO has found applications in various domains, including online advertising \citep{ma2018modeling} and reinforcement learning \citep{thomas2021multi}. However, one of the main challenges in MOO is \textbf{gradient conflict}: objectives with large gradients can dominate the update direction, potentially leading to poor performance for other objectives with smaller gradients. To mitigate this issue, several MOO-based methods have been proposed to balance the contribution of all objectives. A typical method is Multiple Gradient Descent Algorithm (MGDA) \citep{desideri2012multiple}, which seeks a conflict-avoidant update direction that maximizes the minimal improvement across all objectives. MGDA converges to a Pareto stationary point, where no common descent direction exists for all objective functions. Building on MGDA, several other methods have been introduced to further enhance conflict avoidance and improve solution balancing, including CAGrad \citep{liu2021conflict}, PCGrad \citep{yu2020gradient}, GradDrop \citep{chen2020just}.

The work most closely related to ours is MT-SGD \citep{phan2022stochastic}, which is designed for multi-target sampling. In this setup, we are given a set of multiple unnormalized target distributions and aim to generate particles that simultaneously approximate these target distributions. MT-SGD is shown to reduce to multi-objective optimization, where each objective function corresponds to a KL divergence for one of the target distributions. Additionally, MT-SGD has been successfully applied to multi-task learning, achieving state-of-the-art performance across various baselines, thanks to its ability to sample particles from the joint likelihood of multiple target distributions. Inspired by MT-SGD, we introduce the formulation of Multi-Objective Distribution Optimization (MODO). Given a set of objective functionals  $F_{1}(q), F_{2}(q),..., F_{K}(q):\mathcal{P}_{2}(\mathcal{X})\rightarrow \mathbb{R}$, where each $F_{k}(q)$ is defined over the space of probability distribution $q\in \mathcal{P}_{2}(\mathcal{X})$, our goal is to find the optimal distribution that minimizes the following vector-valued objective functional
\begin{equation}  
\label{definition:MODO}
\min_{q\in\mathcal{P}_{2}(\mathcal{X})} \textbf{F}(q) = \min_{q\in\mathcal{P}_{2}(\mathcal{X})} F_{1}(q), F_{2}(q),..., F_{K}(q).
\end{equation}
 Note that, while each $f_{k}$ in (\ref{definition:MOO}) is defined over the space of vectors $\textbf{x}\in\mathcal{X}$, each $F_{k}$ in (\ref{definition:MODO}) is defined over the space of probability distributions $q\in \mathcal{P}_{2}(\mathcal{X})$. 
 Similar to MOO, for two distribution $p,q\in \mathcal{P}_{2}(\mathcal{X})$, we say that $p$ dominates $q$ if $F_k(q)\leq F_k(p)$ for all $k\in [K]$, and $\textbf{F}(p)\neq \textbf{F}(q)$. A distribution $q$ is Pareto optimality if no other distribution in $\mathcal{P}_{2}(\mathcal{X})$ dominates it. 
 
 To solve the MODO problem, we introduce an iterative algorithm, which we call Multi-objective Wasserstein Gradient Descent (MWGraD) by constructing a flow of probability distributions, gradually minimizing all the objective functionals. Specifically, MWGraD consists of two key steps at each iteration. First, for each objective functional $F_{k}$ (for $k\in[K]$), MWGraD estimates the Wasserstein gradient based on the current probability distribution. Second, MWGraD aggregates these gradients into a single Wasserstein gradient using dynamically updated weights and updates the current probability distribution accordingly. In practice, MWGraD operates on a flow of empirical distributions, where each distribution is represented by a set of particles that are updated iteratively. We emphasize that MWGraD can be viewed as a generalized version of MT-SGD \citep{phan2022stochastic}. That is, while MT-SGD specifically tackles the multi-target sampling problem as a form of MODO with the KL divergence as the objective functional, MWGraD can handle a broader class of functionals. Furthermore, we provide theoretical analysis on the convergence of MWGraD to the Pareto stationary point, and experimental results on both synthetic and real-world datasets, demonstrating the effectiveness of the proposed algorithm.

\section{Related Works}
\textbf{Distributional Optimization}. Two widely used Bayesian sampling methods are Gradient Markov chain Monte Carlo (MCMC)  \citep{welling2011bayesian}  and Stein variational gradient descent (SVGD) \citep{liu2016stein}. Gradient MCMC generates samples from a Markov chain to approximate a target distribution (e.g., a posterior), but the resulting samples can be highly correlated. In contrast, SVGD initializes a set of particles and updates them iteratively to approximate the target distribution, often achieving good approximations with relatively fewer samples. As noted by \citep{chen2018unified}, SVGD can be viewed as simulating the steepest descending curves, or gradient flows, of the KL-divergence on a certain kernel-related distribution space. Specifically, the functional $F(q)$ is defined as $KL(q, \pi)$, where $\pi(\textbf{x}) \propto \exp\{-g(\textbf{x})\}$ and $g(\textbf{x})$ is often referred to as the energy function or potential function. Inspired by this perspective, other particle-based variational inference (ParVIs) methods have been developed to simulate the gradient flow in the Wasserstein space. The particle optimization (PO) method \citep{chen2017particle} and the w-SGLD method \citep{chen2018unified} adopt the minimizing movement scheme \citep{jordan1998variational} to approximate the gradient flow using a set of particles. The Blob method \citep{chen2018unified} uses the vector field formulation of the gradient flow and approximates the update direction using particles.
However, when $F$ is not the energy functional and $\pi$ is represented by a set of samples, it becomes non-trivial to define the Wasserstein gradient flow of $F$. To address this challenge, the VT algorithm \citep{liu2021infinite} assumes that $F$ admits a variational form. This assumption allows the Wasserstein gradient to be estimated using samples from both the current empirical distribution and the target distribution, enabling  particle updates in specified directions. Building on VT, MirrorVT \citep{nguyen2023mirror} extends this framework to optimize $F(q)$ when $q$ is defined over a constrained domain. MYVT \citep{nguyen2024moreau} further generalizes the approach to regularized distributional optimization, where the objective is composed of two functionals: one with a variational form and the other expressed as the expectation of a possibly nonsmooth convex function.\\

\noindent
\textbf{Multi-Objective Optimization (MOO)}. Several gradient-based techniques have been proposed for MOO. Among the most popular is MGDA \citep{desideri2012multiple}, which seeks to find a conflict-avoidant update direction that maximizes the minimal improvement across all objectives. PCGrad \citep{yu2020gradient} mitigates the gradient conflict by projecting the gradient of each task on the norm plane of other tasks. GradDrop \citep{chen2020just} randomly drops out conflicted gradients, while CAGrad \citep{liu2021conflict} adds a constraint to ensure the update direction is close to the average gradient. The methods most closely related to our work are MOO-SVGD \citep{liu2021profiling} and MT-SGD \citep{phan2022stochastic}, which enable sampling from multiple target distributions, a task that can be seen as an instance of MODO. MT-SGD aims to update particles in a way that brings them closer to all target distributions, effectively generating diverse particles that lie within the joint high-likelihood region of all targets. In contrast, MOO-SVGD uses MGDA \citep{desideri2012multiple} to update the particles individually and independently. In our synthetic experiments, we observe that the behavior of our proposed algorithm MWGraD closely resembles MT-SGD.

\section{Preliminaries}
\subsection{Basic Concepts of Optimal Transport and Wasserstein Space}
Optimal transport \citep{villani2021topics} has received much attention in the machine learning community and has been shown to be an effective tool for comparing probability distributions in many applications \citep{nguyen2023linear,nguyen2021learning,petric2019got, nguyen2023wasserstein}.
Formally, given a measurable map $T:\mathcal{X}\rightarrow \mathcal{X}$ and $p\in \mathcal{P}_{2}(\mathcal{X})$, we say that $q$ is the \textit{push-forward measure} of $p$ under $T$, denoted by $q=T\sharp p$, if for every Borel set $E\subseteq \mathcal{X}$, $q(E)=p(T^{-1}(E))$. For any $p,q\in \mathcal{P}_{2}(\mathcal{X})$, the $2$-Wasserstein distance $\mathcal{W}_{2}(p,q)$ is defined as
\begin{align*}
   \mathcal{W}_{2}^{2}(p, q)= \inf_{\pi \in \Pi(p,q)} \int_{\mathcal{X}\times \mathcal{X}} \lVert \textbf{x} - \textbf{x}^\prime\rVert_{2}^{2}\mathrm{d}\pi(\textbf{x},\textbf{x}^\prime),
\end{align*}
where $\Pi(p,q)$ is all probability measures on $\mathcal{X}\times \mathcal{X}$ whose two marginals are equal to $p$ and $q$, and $\lVert\cdot \rVert_{2}$ denotes the Euclidean norm. The metric space $(\mathcal{P}_{2}(\mathcal{X}), \mathcal{W}_{2})$, also known as Wasserstein space, is an infinite-dimensional geodesic space \citep[Definition 6.4]{villani2009optimal}. Furthermore, we can endow the manifold $\mathcal{P}_{2}(\mathcal{X})$ with a Riemannian metric \textcolor{blue}{\citep[p.~250]{villani2021topics}}, as follows: for any $s_{1},s_{2}$ are two tangent vectors at $p$, where $s_{1},s_{2} \in \mathcal{T}_{p}\mathcal{P}_{2}(\mathcal{X})$, and $\mathcal{T}_{p}\mathcal{P}_{2}(\mathcal{X})$ denotes the space of tangent vectors at $p$, let $u_{1},u_{2}:\mathcal{X}\rightarrow \mathbb{R}$ be the solutions to the following elliptic equations $s_{1}=- \texttt{div}(\rho \nabla u_{1})$ and $s_{2}=- \texttt{div}(\rho \nabla u_{2})$, respectively, where $\texttt{div}$ denotes the divergence operator on $\mathcal{X}$. The inner product between $s_{1}$ and $s_{2}$ is defined as
\begin{equation*}
    \langle s_{1}, s_{2}\rangle_{p}=\int_{\mathcal{X}}\langle \nabla u_{1}(\textbf{x}),\nabla u_{2}(\textbf{x})\rangle p(\textbf{x})d\textbf{x}.
\end{equation*}

\begin{defn}
    \label{def:firstvariational}
    \textbf{(First variation of a functional)} Given a functional $F:\mathcal{P}_{2}(\mathcal{X})\rightarrow \mathbb{R}$,
    the first variation of $F$ evaluated at $p$, denoted by $\delta F(p):\mathcal{X}\rightarrow \mathbb{R}$, is  given as follows
\begin{align*}
    \lim_{\epsilon\rightarrow 0}\frac{1}{\epsilon} \left( F(p+\epsilon \chi) - F(p) \right)=
    \int_{\mathcal{X}} \delta F(p)(\textbf{x})\chi(\textbf{x})\mathrm{d}\textbf{x},
\end{align*}
for all $\chi=q-p$, where $q\in \mathcal{P}_{2}(\mathcal{X})$.
\end{defn}
\noindent
With mild regularity assumptions, the Wasserstein gradient of $F$, denoted by $\texttt{grad}F$, relates to the gradient of the first variation of $F$ via the following continuity equation
\begin{align}
\label{eqn:grad}
\begin{split}
        &\texttt{grad}F(p)(\textbf{x})=-\texttt{div}\left(p(\textbf{x})\nabla\delta F(p)(\textbf{x})\right),\\
        &\text{ for all }\textbf{x}\in \mathcal{X}.
\end{split}
\end{align}
We refer the readers to \citep{santambrogio2015optimal} for details.

\subsection{Basic Concepts of MOO and MGDA}
For a MOO problem, very often, no single solution can optimize all the objectives at the same time. For instance, it could have two vectors where one performs better for the objective $k$ and the other performs better for the objective $l\neq k$. Thus, Pareto optimality is defined to address such an incomparable case.
\begin{defn}
    \textbf{(Pareto optimality)} For two solutions $\textbf{x},\textbf{y}\in\mathcal{X}$ of (\ref{definition:MOO}), we say that $\textbf{y}$ dominates $\textbf{x}$ if $f_{k}(\textbf{x})\leq f_{k}(\textbf{y})$ for all $k\in[K]$ and $\textbf{f}(\textbf{x}) \neq \textbf{f}(\textbf{y})$. A solution $\textbf{x}$ is Pareto optimal if no other solution in $\mathcal{X}$ dominates it. 
\end{defn}
\noindent
Note that there is a set of Pareto optimal solutions, called Pareto set. Let denote the probability simplex as $\mathcal{W}=\left\{\textbf{w}=(w_{1},...,w_{K})^\top| \textbf{w}\geq 0, \sum_{k=1}^{K}w_{k}=1 \right\}$, we introduce the concept of Pareto Stationarity (also referred to as Pareto Criticality) \citep{custodio2011direct} as follows.

\begin{defn}
    \textbf{(Pareto Stationary Solution)} A solution $\textbf{x}^{*}\in\mathcal{X}$ is a Pareto stationary solution iff some convex combination of the gradients $\left\{\nabla f_{k}(\textbf{x}^{*}) \right\}$ vanishes, i.e., there exists some $\textbf{w}\in\mathcal{W}$ such that $\nabla \textbf{f}(\textbf{x}^{*}) \textbf{w}=\sum_{k=1}^{K}w_{k}\nabla f_{k}(\textbf{x}^{*})=\mathbf{0}$, where $\nabla \textbf{f}(\textbf{x})=\left[\nabla f_{1}(\textbf{x}), f_{2}(\textbf{x}),...,f_{K}(\textbf{x}) \right]$.
\end{defn}
\noindent
MGDA \citep{desideri2012multiple} has gained significant attention in machine learning recently largely because of its gradient-based nature, in contrast to traditional MOO methods. In each iteration, MGDA updates the parameters as follows: $\textbf{x}^{(t+1)}=\textbf{x}^{(t)} - \alpha \textbf{d}^{(t)}$, where $\alpha$ is the learning rate and $\textbf{d}^{(t)}$ is the MGDA search direction at $t$-th iteration. The key idea of MGDA is to find a update direction $\textbf{d}^{(t)}$ that maximizes the minimum decrease across all the objective by solving the following primal problem
\begin{equation*}
\textbf{d}^{(t)}=\argmin_{\textbf{d}\in\mathbb{R}^{d}} \left\{ \max_{k\in [K]} \langle \nabla f_{k}(\textbf{x}^{(t)}), \textbf{d} \rangle + \frac{1}{2}\lVert \textbf{d} \rVert^{2}_{2} \right\}.
\end{equation*}
\noindent
To simplify the optimization, such primal problem has a dual objective as a min-norm oracle
\begin{equation*}
    \textbf{w}^{(t)}=\argmin_{\textbf{w}\in\mathcal{W}} \lVert \sum_{k=1}^{K}w_{k}\nabla f_{k}(\textbf{x}^{(t)}) \rVert^{2}_{2}.
\end{equation*}
The update direction is then calculated by $\textbf{d}^{(t)}=\nabla \textbf{f}(\textbf{x}^{(t)}) \textbf{w}^{(t)}= \sum_{k=1}^{K}w_{k}^{(t)}\nabla f_{k}(\textbf{x}^{(t)})$.

\section{Multiple Wasserstein Gradient Descent Algorithm }
In this section, we introduce MWGraD, an iterative algorithm for solving the MODO problem (\ref{definition:MODO}). Specifically, we construct a flow of distributions $q^{(0)}$, $q^{(1)}$,..., $q^{(T)}$, starting with a simple distribution $q^{(0)}$, such as  standard normal distribution, and progressively minimizing all the objective functionals $\left(F_{1}(q), F_{2}(q),...,F_{K}(q) \right)$ simultaneously. We begin by reformulating the MODO problem, then present the method for aggregating multiple Wasserstein gradients, followed by the detailed descriptions of algorithms.
\subsection{Reformulation of MODO and Multiple Wasserstein Gradient Aggregation Method}
To construct the sequence $\{q^{(t)}\}_{t\geq 0}$, we first reformulate the MODO problem (\ref{definition:MODO}) as follows. For any $q\in \mathcal{P}_{2}(\mathcal{X})$ and any tangent vector $s\in \mathcal{T}_{q}\mathcal{P}_{2}(\mathcal{X})$ at $q$, let $\gamma:[0,1]\rightarrow \mathcal{P}_2(\mathcal{X})$ be a curve satisfying $\gamma(0)=q$ and $\gamma^\prime(0)=s$. By the definition of directional derivative, we have that
\begin{equation}
    \lim_{h\rightarrow 0}\frac{1}{h}\left[F_{k}(\gamma(h)) - F_{k}(q) \right]=\int_{\mathcal{X}} \delta F_{k}(q)(\textbf{x})s(\textbf{x})d\textbf{x}. 
\end{equation}
Let $\textbf{u}:\mathcal{X}\rightarrow \mathcal{X}$ be a vector field, where $\textbf{u}\in\mathcal{V}$, and $\mathcal{V}$ denotes the space of velocity fields $\textbf{u}$. Assume that $\textbf{u}$ satisfies the following elliptic equation
\begin{equation}
    s(\textbf{x}) + \text{div}(q(\textbf{x})\textbf{u}(\textbf{x})) = 0, \forall \textbf{x}\in \mathcal{X}.
\end{equation}
By the integration by parts, we obtain
\begin{align}
\label{eqn:1}
\begin{split}
    &\lim_{h\rightarrow 0}\frac{1}{h}\left[F_{k}(\gamma(h)) - F_{k}(q) \right] =\\
    &-\int_{\mathcal{X}} \delta F_{k}(q)(\textbf{x})\text{div}(q(\textbf{x})\textbf{u}(\textbf{x}))d\textbf{x} \\
    &= -\int_{\mathcal{X}} \texttt{div}\left(\delta F_{k}(q)(\textbf{x}) \textbf{u}(\textbf{x}) q(\textbf{x})\right)d\textbf{x}\\
    &+\int_{\mathcal{X}} \langle\nabla \delta F_{k}(q)(\textbf{x}), \textbf{u}(\textbf{x})\rangle q(\textbf{x})d\textbf{x}.
\end{split}
\end{align}
By the divergence theorem \citep{rudin2021principles}, it holds that the first term on the right-hand side is equal to zero. Thus, we obtain that
\begin{align}
\label{eqn:2}
\begin{split}
    &\lim_{h\rightarrow 0}\frac{1}{h}\left[F_{k}(\gamma(h)) - F_{k}(q) \right] \\
    &=\int_{\mathcal{X}} \langle\nabla \delta F_{k}(q)(\textbf{x}), \textbf{u}(\textbf{x})\rangle q(\textbf{x})d\textbf{x}.
\end{split}
\end{align}

We can similarly rewrite (\ref{eqn:2}) for the opposite direction of the update as follows
\begin{align}
\label{eqn:3}
\begin{split}
    &\lim_{h\rightarrow 0}\frac{1}{h}\left[F_{k}(q) - F_{k}(\gamma(h)) \right]\\
    &=\int_{\mathcal{X}} \langle \nabla \delta F_{k}(q)(\textbf{x}), \textbf{v}(\textbf{x})\rangle q(\textbf{x})d\textbf{x}, 
\end{split}
\end{align}
where $\textbf{v}\in\mathcal{V}$ and $\textbf{v}(\textbf{x})=-\textbf{u}(\textbf{x}),\forall \textbf{x}\in\mathcal{X}$. In the MODO problem (\ref{definition:MODO}), there are multiple potentially conflicting objectives $\left\{ F_{k}\right\}_{k=1}^{K}$. Inspired by \citep{desideri2012multiple}, we aim to estimate a tangent vector $s^{(t)}$, for the $t$-th iteration, that minimizes each objective $F_{k}$. Based on (\ref{eqn:3}), we reformulate the problem as follows: we aim to find $s^{(t)}$ that maximizes the minimum decrease across all the objectives
\begin{align}
\label{eqn:4}
\begin{split}
    \max_{s\in \mathcal{T}_{q^{(t)}}\mathcal{P}_{2}(\mathcal{X})}\min_{k\in \left[K\right]}\frac{1}{h}\left(F_{k}(q^{(t)})- F_{k}(\gamma(h)) \right)\\
    \approx \max_{\textbf{v}\in \mathcal{V}}\min_{k\in \left[K\right]}\int_{\mathcal{X}} \langle \nabla \delta F_{k}(q^{(t)})(\textbf{x}), \textbf{v}(\textbf{x})\rangle q^{(t)}(\textbf{x})d\textbf{x}, 
\end{split}
\end{align}
where the approximation is based on (\ref{eqn:3}) and $\gamma:[0,1]\rightarrow \mathcal{P}_2(\mathcal{X})$ is a curve satisfying that $\gamma(0)=q^{(t)}$ and $\gamma^\prime(0)=s$. To regularize the update direction (i.e., vector field $\textbf{v}$), we introduce a regularization term to (\ref{eqn:4}) and solve for $s^{(t)}$ by optimizing
\begin{align}
\label{optimizationproblem}
\begin{split}
   &\max_{\textbf{v}\in \mathcal{V}}\min_{k\in \left[K\right]}\int_{\mathcal{X}} \langle \nabla \delta F_{k}(q^{(t)})(\textbf{x}), \textbf{v}(\textbf{x})\rangle q^{(t)}(\textbf{x})d\textbf{x} - \\
   &\frac{1}{2}\int_{\mathcal{X}} \lVert \textbf{v}(\textbf{x})\rVert^{2}_{2}q^{(t)}(\textbf{x})d\textbf{x}.
\end{split}
\end{align}
\noindent
The following theorem provides the solution to problem (\ref{optimizationproblem}).
\begin{thm}
\label{theorem1}
    Problem (\ref{optimizationproblem}) has a solution $\textbf{v}^{(t)}$ as follows. For $\textbf{x}\in \mathcal{X}$, we have that
    \begin{equation}
    \label{eqn:vt}
        \textbf{v}^{(t)}(\textbf{x})= \textbf{V}^{(t)}(\textbf{x}) \textbf{w}^{*}=\sum_{k=1}^{K}w^{*}_{k}\textbf{v}^{(t)}_{k}(\textbf{x}),
    \end{equation}
    where $\textbf{v}^{(t)}_{k}(\textbf{x}) = \nabla \delta F_{k}(q^{(t)})(\textbf{x})$ for $k\in[K]$, $\textbf{V}^{(t)}(\textbf{x})= \left[\textbf{v}^{(t)}_{1}(\textbf{x}), \textbf{v}^{(t)}_{2}(\textbf{x}),...,\textbf{v}^{(t)}_{K}(\textbf{x}) \right]$, and 
    \begin{equation}
        \textbf{w}^{*} = \argmin_{\textbf{w}\in\mathcal{W}}\frac{1}{2}\int_{\mathcal{X}}\lVert \textbf{V}^{(t)}(\textbf{x}) \textbf{w} \rVert^{2}_{2}q^{(t)}(\textbf{x})d\textbf{x}.
    \label{problem:solveweights}
    \end{equation}
\end{thm}
\noindent
The proof follows from the method of Lagrange multipliers and is detailed in Appendix \ref{appendix:proofoftheorem1}. Note that the velocity used to update the current particles from $q^{(t)}$ is computed as a weighted sum of the velocities corresponding to each of $K$ objective functionals. The weights are obtained by solving the min-norm oracle in Theorem \ref{theorem1}. 

\subsection{Algorithm and Implementation}
To make the solution for $\textbf{w}$ in the min-norm oracle (Theorem \ref{theorem1}) computationally feasible, we approximate $q^{(t)}$ using $m$ particles $\left\{\textbf{x}^{(t)}_{i} \right\}_{i=1}^{m}$. Each particle is updated as $\textbf{x}^{(t+1)}_{i} = \textbf{x}^{(t)}_{i} - \alpha \textbf{v}^{(t)}(\textbf{x}^{(t)}_{i} )$, where $\alpha$ is the step size.
The velocity fields $\textbf{v}^{(t)}_{k}=\nabla \delta F_{k}(q^{(t)})$ for $k\in[K]$ need to be computed at each step, but exact computation is difficult, so we focus on approximation methods for $\textbf{v}^{(t)}_{k}$ for two specific forms of $F_{k}(q)$.\\

\noindent
\textbf{Energy Functional}. Consider $F_{k}(q)$ to be defined as
\begin{align}
\label{energyfunctional}
\begin{split}
    F_{k}(q)&=  \int_{\mathcal{X}} g_{k}(\textbf{x})q(\textbf{x})d\textbf{x} + \int_{\mathcal{X}} \log q(\textbf{x})q(\textbf{x})d\textbf{x}\\
    &= KL(q|| \exp{\left\{-g_{k}\right\}}).
\end{split}
\end{align}
\noindent
This form of energy functional is commonly used in Bayesian learning, where the goal is to approximate the posterior distribution with $q$. It is straightforward to verify that the first variation of $F_{k}(q)$ is $\delta F_{k}(q)(\textbf{x})= g_{k}(\textbf{x})+ \log q(\textbf{x}) + 1$. Thus, the velocity $\textbf{v}^{(t)}_{k}(\textbf{x})$ can be computed as the gradient of $\delta F_{k}(q^{(t)})(\textbf{x})$. However, directly applying the particle-based approximation is infeasible because the term $\log q^{(t)}(\textbf{x})$ is undefined with discrete representations of $q^{(t)}(\textbf{x})$. To address this issue, we present two commonly used techniques for approximating $\textbf{v}^{(t)}_{k}(\textbf{x})$, SVGD \citep{liu2016stein}, and Blob methods \citep{carrillo2019blob}, both of which are kernel-based.\\

\noindent
We apply the idea of \textbf{SVGD} \citep{liu2016stein} to approximate $\textbf{v}^{(t)}_{k}(\textbf{x})$ with $\Tilde{\textbf{v}}_{k}^{(t)}(\textbf{x})$, given by
\begin{equation*}
    \Tilde{\textbf{v}}_{k}^{(t)}(\textbf{x}) =\mathbb{E}_{\textbf{y}\sim q^{(t)}}\left[K(\textbf{x},\textbf{y})\left(\nabla g_{k}(\textbf{y})+\nabla \log q^{(t)}(\textbf{y})\right) \right]. 
\end{equation*}
\noindent
Applying integration by parts, we obtain
\begin{align}
\begin{split}
    &\Tilde{\textbf{v}}_{k}^{(t)}(\textbf{x}) =\\
    &=  \int_{\mathcal{X}} K(\textbf{x}, \textbf{y})\nabla g_{k}(\textbf{y})q^{(t)}(\textbf{y})d\textbf{y} + \int_{\mathcal{X}} K(\textbf{x},\textbf{y})\nabla q^{(t)}(\textbf{y})d\textbf{y}\\
    & =  \int_{\mathcal{X}} K(\textbf{x}, \textbf{y})\nabla g_{k}(\textbf{y})q^{(t)}(\textbf{y})d\textbf{y} - \int_{\mathcal{X}} \nabla_{\textbf{y}} K(\textbf{x},\textbf{y}) q^{(t)}(\textbf{y})d\textbf{y}.
    \end{split}
\end{align}
\noindent
Thus, the particle approximation of $\textbf{v}^{(t)}_{k}$ becomes
\begin{equation}
\label{approximate-svgd}
    \Tilde{\textbf{v}}_{k}^{(t)}(\textbf{x}) =  \sum_{j=1}^{m}K(\textbf{x}, \textbf{x}^{(t)}_{j})\nabla g_{k}(\textbf{x}^{(t)}_{j}) - \sum_{j=1}^{m}\nabla_{\textbf{x}^{(t)}_{j}}K(\textbf{x}^{(t)}_{i}, \textbf{x}^{(t)}_{j})
\end{equation}

\noindent
We can also apply the idea of \textbf{Blob methods} \citep{carrillo2019blob} to smooth the second term of the energy functional by the kernel function $K$, and approximate the velocity field. Specifically, the particle approximation for $\textbf{v}^{(t)}_{k}$ is given by
\begin{align}
\label{approximate-blob}
\begin{split}
     \Tilde{\textbf{v}}_{k}^{(t)}(\textbf{x}) &=  \nabla g_{k}(\textbf{x}) - \sum_{j=1}^{m}\nabla_{\textbf{x}^{(t)}_{j}} K(\textbf{x}, \textbf{x}^{(t)}_{j})/ \sum_{l=1}^{m}K(\textbf{x}^{(t)}_{j}, \textbf{x}^{(t)}_{l})\\
    &- \sum_{j=1}^{m}\nabla_{\textbf{x}^{(t)}_{j}}K(\textbf{x}, \textbf{x}^{(t)}_{j})/\sum_{l=1}^{m}K(\textbf{x}, \textbf{x}^{(t)}_{j}).
\end{split}
\end{align}
For the detailed derivation of the update, see Proposition 3.12 in \citep{carrillo2019blob}.\\

\noindent
\textbf{Dissimilarity Functions}. Let $F_{k}(q)$ be defined as a dissimilarity function $D$ between $q$ and the target distribution $\pi_{k}$, characterized by a set of samples. Common dissimilarity functions include KL divergence and JS divergence. In this case, estimating the first variation $\delta F_{k}(q)$ is not straightforward. To address this issue, following \citep{liu2021infinite}, we assume that $F_{k}(q)$ has the variational form as follows:
\begin{align}
\label{eqn:varform}
    F_{k}(q) = D(q, \pi_{k})=\sup_{h_{k}\in\mathcal{H}}{\left\{ \mathbb{E}_{\textbf{x}\sim q}\left[ h_{k}(\textbf{x})\right]-F^{*}_{k}(h_{k})\right\}},
\end{align}
where $\mathcal{H}$ is a class of square-integrable functions on $\mathcal{X}$ with respect to the Lebesgue measure, and $F_{k}^{*}:\mathcal{H}\rightarrow \mathbb{R}$ is a convex conjugate functional of $F_{k}$, for $k\in[K]$. Note that the variational form of $F_{k}$ involves a supremum over a function class $\mathcal{H}$. 
We restrict $\mathcal{H}$ to a subset of the square-integrable functions, such as a class of deep neural networks or a Reproducing kernel Hilbert space (RKHS), which allows for a numerically feasible maximization process. More importantly, it is shown in \citep{liu2021infinite} that the optimal solution $h_{k}^{*}$ to the problem (\ref{eqn:varform}) is the first variation of $F_{k}$, i.e. $h_{k}^{*}=\delta F_{k}(q)$. As an example, when $D_{k}$ is the KL divergence, its variational form is
\begin{align}
\label{eqn:varformKL}
    KL(q,\pi_{k}) = \sup_{h_{k}\in\mathcal{H}}{\left\{\mathbb{E}_{\textbf{x}\sim q}\left[ h_{k}(\textbf{x})\right]-\log\mathbb{E}_{\textbf{x}\sim \pi_{k}} \left[e^{h_{k}(\textbf{x})}\right]\right\}}.
\end{align}
We can parameterize $h_{k}$ using a neural network $h_{\theta_{k}}$, where $\theta_{k}$ denotes its parameters. The parameters can be estimated using stochastic gradient descent to maximize the following empirical objective function
\begin{equation}
    \max_{\theta_{k}} \frac{1}{m}\sum_{i=1}^{m}h_{\theta_{k}}(\textbf{x}^{(t)}_{i}) - \log \frac{1}{n}\sum_{j=1}^{n}\exp (h_{\theta_{k}}(\textbf{y}_{k,j})),
\end{equation}
where the target $\pi_{k}$ is represented by the set of samples $\left\{\textbf{y}_{k,j} \right\}_{j=1}^{n}$. The neural network $h_{\theta^{*}_{k}}$ after learning can be used to estimate the velocity corresponding to $F_{k}$ by taking the gradient of the network with respect to its input:
\begin{equation}
\label{approximate-nn}
    \Tilde{\textbf{v}}^{(t)}_{k}(\textbf{x}^{(t)}_{i})=\nabla h_{\theta^{*}_{k}}(\textbf{x}^{(t)}_{i}), \text{ for } i\in [m], k\in[K].
\end{equation}

\noindent
\textbf{Update of $\textbf{w}^{(t)}$}. Following the approach of \citep{zhang2024convergence}, instead of computing the exact solution of $\textbf{w}$ to the optimization problem in Theorem \ref{theorem1}, we approximate it by taking one step of gradient descent and then use the updated weights to aggregate the velocity fields. Specifically, we update each particle $\textbf{x}^{(t)}_{i}$ along the following direction
\begin{equation}
    \Tilde{\textbf{v}}^{(t)}(\textbf{x}^{(t)}_{i})= \sum_{k=1}^{K}w_{k}^{(t)}\Tilde{\textbf{v}}_{k}^{(t)}(\textbf{x}_{i}^{(t)}),
\end{equation}
and update the weights $\textbf{w}^{(t)}$ as follows
\begin{align}
\label{eqn:updatew}
\begin{split}
    &\textbf{w}^{(t+1)} = \\
    &\Pi_{\mathcal{W}}\left(\textbf{w}^{(t)} - \beta \left[ \sum_{i=1}^{m} \left(\Tilde{\textbf{V}}^{(t)}(\textbf{x}_{i}^{(t)})\right)^\top\Tilde{\textbf{V}}^{(t)}(\textbf{x}_{i}^{(t)})\right]\textbf{w}^{(t)}\right),
\end{split}
\end{align}
where $\Tilde{\textbf{V}}^{(t)}(\textbf{x})=\left[\Tilde{\textbf{v}}^{(t)}_{1}(\textbf{x}),\Tilde{\textbf{v}}^{(t)}_{2}(\textbf{x}),...,\Tilde{\textbf{v}}^{(t)}_{K}(\textbf{x}) \right]$ is the approximation of $\textbf{V}^{(t)}(\textbf{x})$, $\Pi_{\mathcal{W}}$ is the projection operator on the simplex $\mathcal{W}$, and $\beta$ is the step size for updating $\textbf{w}$. Taking all into account, we have Algorithm \ref{alg:mwgrad} (see Appendix \ref{appendix:algorithm}).

\subsection{Convergence Analysis}
In this section, we analyze the convergence of MWGrad. We first define the Pareto stationary points in the space of probability distributions, and then characterize how MWGraD can converge to the Pareto stationary points. 

\begin{defn}
    $q\in \mathcal{P}_{2}(\mathcal{X})$ is a Pareto stationary point if 
    \begin{equation*}
        \min_{\textbf{w}\in\mathcal{W}}\langle\texttt{grad} \textbf{F}(q)\textbf{w}, \texttt{grad} \textbf{F}(q)\textbf{w}\rangle_{q}=0,
    \end{equation*}
    where $\texttt{grad} \textbf{F}(q)(\textbf{x})=\left[\texttt{grad} F_{1}(q)(\textbf{x}),\texttt{grad} F_{2}(q)(\textbf{x}),...,\texttt{grad} F_{k}(q)(\textbf{x}) \right]$, and $\texttt{grad} \textbf{F}(q)(\textbf{x})\textbf{w}=\sum_{k=1}^{K}w_{k}\texttt{grad}F_{k}(q)(\textbf{x})$. Further, we call $q$ an $\epsilon$-accurate Pareto stationary distribution if 
    \begin{equation*}
        \min_{\textbf{w}\in\mathcal{W}}\langle\texttt{grad} \textbf{F}(q)\textbf{w}, \texttt{grad} \textbf{F}(q)\textbf{w}\rangle\leq \epsilon^{2}.
    \end{equation*}
A detailed discussion on Pareto stationary distribution can be found in Appendix \ref{appendix:pareto}.
\end{defn}
\noindent
As previously discussed, we need to approximate the true velocity field $\textbf{v}^{(t)}_{k}$ with $\Tilde{\textbf{v}}^{(t)}_{k}$ for $k\in [K]$. The true Wasserstein gradient of $F_{k}$ at $q^{(t)}$ is given by $\texttt{grad} F_{k}(q^{(t)})= -\texttt{div}(q^{(t)}\textbf{v}^{(t)}_{k})$, while its estimate is given by $-\texttt{div}(q^{(t)}\Tilde{\textbf{v}}^{(t)}_{k})$. The deviation between them is expressed as $\xi^{(t)}_{k}=-\texttt{div}(q^{(t)}(\Tilde{\textbf{v}}^{(t)}_{k} - \textbf{v}^{(t)}_{k}))\in \mathcal{T}_{q^{(t)}}\mathcal{P}_{2}(\mathcal{X})$. The gradient error is then defined as
\begin{align}
    \epsilon^{(t)}_{k}=\langle \xi^{(t)}_{k}, \xi^{(t)}_{k} \rangle_{q^{(t)}}=\int_{\mathcal{X}} \lVert  \Tilde{\textbf{v}}_{k}^{(t)}(\textbf{x})- \textbf{v}^{(t)}_{k}(\textbf{x}) \rVert^{2}_{2}q^{(t)}(\textbf{x})\mathrm{d}\textbf{x},
\end{align}
In addition, we assume that the gradient error is upper bounded by a constant $\sigma > 0$ for $k\in[K]$ and $t\geq 0$.

\noindent
\textbf{Assumption 1} (\textbf{Wasserstein gradient error}). We assume that there is a constant $\sigma>0$ such that 
\begin{equation*}
    \langle \xi^{(t)}_{k}, \xi^{(t)}_{k} \rangle_{q^{(t)}}=\epsilon^{(t)}_{k} \leq \sigma^{2}.
\end{equation*}
We further make the following assumption on the functionals $F_{k}$ to analyze the convergence of MWGraD. 

\noindent
\textbf{Assumption 2} (\textbf{Geodesic smoothness}). We assume that $F_{k}$ is geodesically $\ell_{k}$-smooth with respect to the $2$-Wasserstein distance, for $k\in[K]$, in the sense that for $\forall p$, $q\in \mathcal{P}_{2}(\mathcal{X})$
\begin{align}
\begin{split}
     F_{k}(q) &\leq F_{k}(p) + \langle \texttt{grad}F_{k}(p), \texttt{Exp}_{p}^{-1}(q)\rangle_{p} \\
     &+ \frac{\ell_{k}}{2}\cdot \mathcal{W}^{2}_{2}(p,q),
\end{split}
\end{align}
where $\texttt{Exp}_{p}$ denotes the exponential mapping, which specifies how to move $p$ along a tangent vector on $\mathcal{P}_{2}(\mathcal{X})$ and $\texttt{Exp}_{p}^{-1}$ denotes its inversion mapping, which maps a point on $\mathcal{P}_{2}(\mathcal{X})$ to a tangent vector. We refer the readers to \citep{santambrogio2015optimal} for more details. 
We then have the following theorem for the convergence of MWGraD.

\begin{thm}
\label{theorem:convergence_simple}
    Let Assumptions 1 and 2 hold, and $\epsilon > 0$ be a small constant. Set $\alpha \leq \mathcal{O}(\epsilon^{2})$, $\beta \leq \mathcal{O}(\epsilon^{2})$, $T\geq \max \left\{ \Theta(\frac{1}{\alpha \epsilon^{2}}), \Theta(\frac{1}{\beta \epsilon^{2}}) \right\}$. We then have that
    \begin{align}
    \label{eqn:graddientnorm}
    \begin{split}
        &\min_{0\leq t\leq T-1}\langle \texttt{grad} \textbf{F}(q^{(t)})\textbf{w}^{(t)}, \texttt{grad} \textbf{F}(q^{(t)})\textbf{w}^{(t)} \rangle_{q^{(t)}}  \\
        &\leq \frac{1}{T}\sum_{t=0}^{T-1}\langle \texttt{grad} \textbf{F}(q^{(t)})\textbf{w}^{(t)}, \texttt{grad} \textbf{F}(q^{(t)})\textbf{w}^{(t)} \rangle_{q^{(t)}} \\
        &\leq \mathcal{O}(\epsilon^{2}) + 3\sigma^{2}.
    \end{split}
    \end{align}
\end{thm}
\noindent
The formal version of Theorem \ref{theorem:convergence_simple} and its detailed proof can be found in Appendix \ref{appendix:proofoftheorem2}. Theorem \ref{theorem:convergence_simple} indicates that when we set $\alpha = \beta = \mathcal{O}(\epsilon^{2})$ and $T\geq \Theta(\epsilon^{-4})$, we can find an $\sqrt{\mathcal{O}(\epsilon^{2})+3\sigma^{2}}$-accurate Pareto stationary point by running MWGraD. Furthermore, we can see that the squared norm of the convex combination of Wasserstein gradients in Theorem \ref{theorem:convergence_simple} is upper bounded by a sum of two terms. The first term $\mathcal{O}(\epsilon^{2})$, which can be arbitrarily small, corresponds to the convergence rate of MWGraD with the exact velocity fields. Meanwhile, the second term $3\sigma^{2}$ exhibits the effect of the gradient error caused by the approximation of velocity fields.

\section{Experimental Results}
In this section, we present numerical experiments on both synthetic and real-world datasets to demonstrate the effectiveness of MWGraD. The code can be found at \url{https://github.com/haidnguyen0909/MWGraD}.
\subsection{Experiments on Synthetic Datasets}

\noindent
\textbf{Energy Functional}. We consider $F_{k}(q)$ as an energy functional (\ref{energyfunctional}), related to sampling from multiple target distributions. Each target distribution is a mixture of two Gaussian distributions
$\pi_{k}(\textbf{x})=\eta_{k1}\mathcal{N}(\textbf{x}|\mu_{k1}, \Sigma_{k1}) + \eta_{k2}\mathcal{N}(\textbf{x}|\mu_{k2}, \Sigma_{k2})$, $k=1,2,3,4$, where $\eta_{k1}=0.7$, $\eta_{k2}=0.3$ for $k=1,2,3,4$, the means $\mu_{11}=[4, -4]^\top$, $\mu_{12}=[0, 0.1]^\top$, $\mu_{21}=[-4, 4]^\top$, $\mu_{22}=[0, -0.1]^\top$, $\mu_{31}=[-4, -4]^\top$, $\mu_{32}=[0.1, 0]^\top$, $\mu_{41}=[4, 4]^\top$, $\mu_{42}=[-0.1, 0]^\top$, and the common covariance matrix $\Sigma_{kj}$ is the identity matrix of size $2\times 2$, for $k=1,2,3,4$ and $j=1,2$. 
The distribution $q$ is represented by 50 particles, initially sampled from the standard distribution. We update the particles using MOO-SVGD \citep{liu2021profiling} and variants of our MWGraD, including MWGraD with SVGD (\ref{approximate-svgd}) (denoted as MWGraD-SVGD) and Blob (\ref{approximate-blob}) (denoted as MWGraD-Blob) methods to approximate the velocities. Figure \ref{fig:toy-1} shows a common high-density region around the origin.

\begin{figure*}[t]
	\centerline{\includegraphics[width=1.0\textwidth]{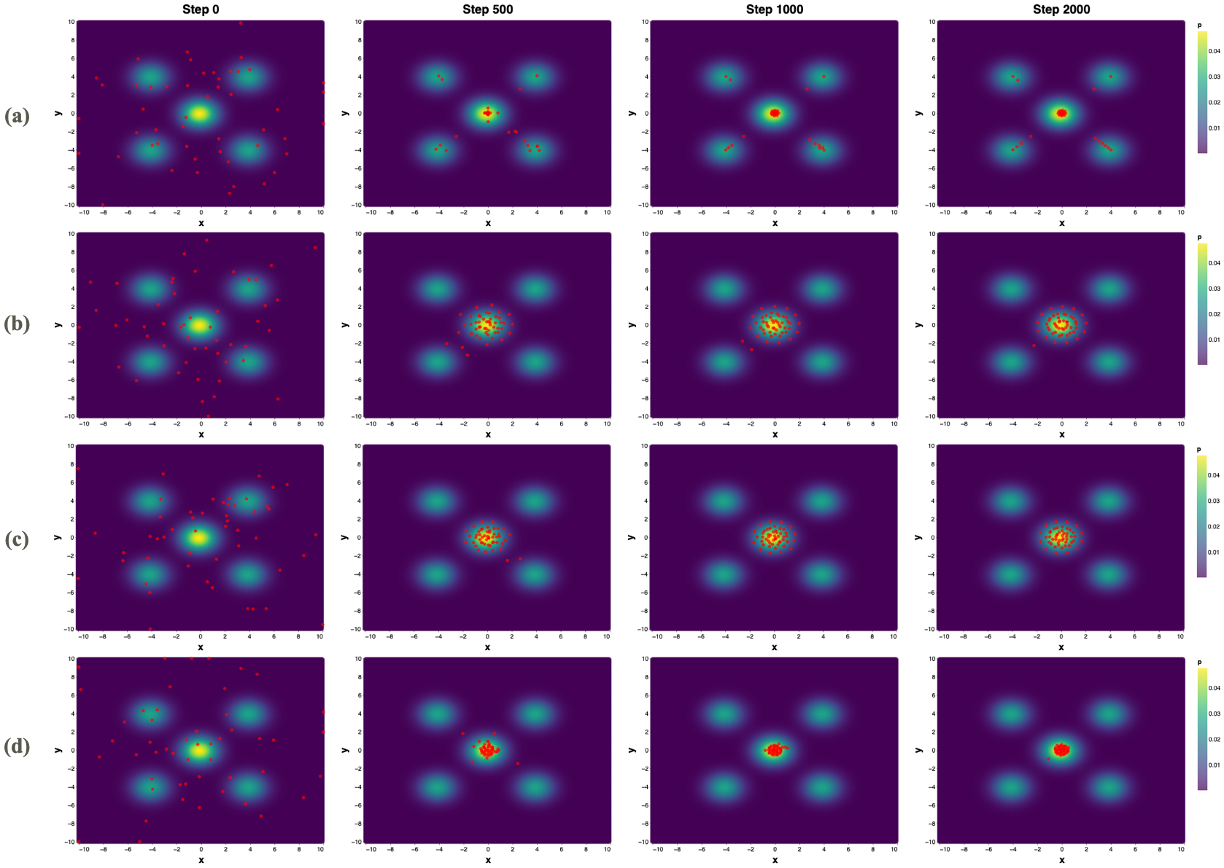}}
	\caption{Sampling from multiple target distributions, where each target is a mixture of two Gaussians. These targets have a joint high-density region around the origin. Initially, 50 particles are sampled from the standard distribution, and then updated using (a) MOO-SVGD and variants of MWGraD, including (b) MWGraD-SVGD, (c) MWGraD-Blob and (d) MWGraD-NN. While MOO-SVGD tends to scatter particles across all the modes, MWGraD tends to move particles towards the joint high-density region.}
	\label{fig:toy-1}
\end{figure*}

We also consider the variational form (\ref{eqn:varform}) of KL divergence for sampling, where $F_{k}$ is expressed as $KL(q, \pi_{k})$. As noted earlier, the variational form of $F_{k}$ can handle cases where $\pi_{k}$ is characterized by samples. We can extend this form to the sampling task by expressing $F_{k}$ as an energy functional by introducing the change of variables $h_{k}^\prime(\textbf{x})=\exp\left\{h_{k}(\textbf{x})\pi_{k}(\textbf{x})/p(\textbf{x})\right\}$, where $p(\textbf{x})$ is a distribution that is easy to sample from, such as $p(\textbf{x})=\mathcal{N}(\textbf{x}|0, \textbf{I})$. Then, the variational form of $KL(q,\pi_{k})$ can be rewritten as
\begin{align}
\label{eqn:changeofvariable}
\begin{split}
&KL(q,\pi_{k})=\\
&\max_{h_{k}^\prime\in\mathcal{H^{+}}}\left\{\mathbb{E}_{\textbf{x}\sim q}\left[\log\left( \frac{h_{k}^\prime (\textbf{x})p(\textbf{x})}{\pi_{k}(\textbf{x})}\right)\right]  
        - \log \mathbb{E}_{\textbf{x}\sim p}[h_{k}^\prime (\textbf{x})]\right\},
\end{split}
\end{align}
where $\mathcal{H}^{+}$ is the space of positive functions. As we can use $\log$ of unnormalized density of $\pi_{k}$ in the above optimization problem, $h_{k}^{\prime}$ can be estimated using samples drawn from $q$ and $p$. We optimize (\ref{eqn:changeofvariable}) by parameterizing $h_{k}^\prime$ with a neural network with two layers, each of which has 50 neurons. We use the ReLU activation function in the last layer to guarantee the output  of the neural network to be positive. 
Furthermore, $h_{k}$ can be estimated from $h_{k}^{\prime}$ using: $h_{k}(\textbf{x})=\log h_{k}^{\prime}(\textbf{x})+\log p(\textbf{x})-\log \pi(\textbf{x})$. 
The parameters of the neural network are trained by running 20 steps of gradient ascent to optimize the objective (\ref{eqn:changeofvariable}). After training, the estimate $h_{k}^{*}$ can be used to approximate the first variation of $F_{k}(q)$, i.e., $\Tilde{\textbf{v}}^{(t)}_{k}=\nabla  h^{*}_{k}$. We denote this method by MWGraD-NN as we use the neural network to approximate the velocities. For all compared methods, we set the step size $\alpha=0.0001$, while for variants of MWGraD, we set the step size $\beta = 0.001$. For MWGraD-SVGD and MWGraD-Blob, we use RBF as kernel $K$, i.e. $K(\textbf{x},\textbf{y})= \exp\left\{-\gamma\lVert \textbf{x}-\textbf{y}\rVert^{2}_{2}\right\}$, where we fix $\gamma$ as 0.01. We run 2000 updates of particles.

\begin{table*}[t]
  \centering
  \begin{tabular}{l l l l l l l l}
  \hline
    \multirow{2}{*}{Datasets} & \textbf{Tasks} & MGDA & MOO-SVGD & MT-SGD & MWGraD & \textbf{MWGraD} & MWGraD \\
    & & & & & -SVGD & -Blob & -NN\\
    \hline
    \multirow{2}{*}{Multi-Fashion+MNIST} &  \#1 & 94.4$\pm$0.6 & 94.8$\pm0.4$ & 96.2$\pm$0.3 & 95.7$\pm$0.4 & \textbf{96.7}$\pm$0.5 & 95.9$\pm$0.4\\ 
    &  \#2 & 85.5$\pm$0.5 & 85.6$\pm$0.2 & 87.8$\pm$0.6 & 88.9$\pm$0.6 & \textbf{92.5}$\pm$0.4 & 88.2$\pm$0.3\\
    \hline
    \multirow{2}{*}{Multi-MNIST} &  \#1 & 93.4$\pm$0.4 & 93.1$\pm$0.3 & 94.4$\pm$0.5 & 94.5$\pm$0.4 & \textbf{97.6}$\pm$0.2 & \textbf{97.7}$\pm$0.5\\ 
    &  \#2 & 91.8$\pm$0.6 & 91.2$\pm$0.2 & 92.9$\pm$0.5 & 93.2$\pm$0.6 & \textbf{96.7}$\pm$0.5 & 95.5$\pm$0.4\\ 
    \hline
    \multirow{2}{*}{Multi-Fashion} &  \#1 & 84.1$\pm$0.8 & 83.8$\pm$0.8 & 84.9$\pm$0.6 & 85.1$\pm$0.7 & 86.8$\pm$0.3 & \textbf{87.2}$\pm$0.4\\ 
    &  \#2 & 83.3$\pm$0.4 & 83.1$\pm$0.3 & 84.6$\pm$0.5 & 84.3$\pm$0.4 & \textbf{87.2}$\pm$0.5 & 85.3$\pm$0.6\\ 
  \end{tabular}
  \caption{Experimental results on Multi-Fashion+MNIST, Multi-MNIST, and Multi-Fashion. We report the ensemble accuracy (higher is better) averaged over three independent runs with different initializations.}
  \label{tab:results}
\end{table*}

Figure \ref{fig:toy-1} shows the particles updated by MOO-SVGD and variants of MWGraD, including MWGraD-SVGD, MWGraD-Blob, and MWGraD-NN at selected iterations. We observe that particles by MOO-SVGD spread out and tend to cover all the modes of the targets, some of them even scatter along the trajectory due to the conflict in optimizing multiple objectives. In contrast, variants of MWGraD tend to cover the common high-density region with particles. \\

\noindent
\textbf{Dissimilarity Functions}. Next we consider $F_{k}(q)=D(q,\pi_{k})$, where $D$ is the KL divergence or JS divergence, and $\pi_{k}$ is characterized by a set of samples. We use the same target distribution used in the case of energy functional but we generate a set of samples to represent the target distribution instead of its energy functional. In our experiments, we generate 30 samples for each target distribution, so we have 120 samples for all objectives. For this case, MOO-SVGD, MWGraD-SVGD and MWGraD-Blob cannot be used. Thus we need to use the variational form of the KL and JS divergences. The detailed experimental settings are described in Appendix \ref{exp:settings}. Figure \ref{fig:toy-2} shows the particles updated by MWGraD at selected iterations, for two cases of KL divergence and JS divergence. We observe that MWGraD tend to cover the common high-density regions with particles, which is consistent with our observation in cases where $F_{k}(q)$ is represented by an energy functional.

\subsection{Experiments on Multi-task learning}
We follow the experimental settings outlined in \citep{phan2022stochastic} to verify the performance of MWGraD.\\

\noindent
\textbf{Multi-task Learning}. We assume to have $K$ prediction tasks and a training dataset $\mathbb{D}$. For each task $k\in[K]$, the model is represented by the parameter vector $\theta^{k}=[\textbf{x}, \textbf{z}^{k}]$, where $\textbf{x}$ is the shared part of the model, and $\textbf{z}^{k}$ is the task-specific non-shared part. The approach outlined in \citep{phan2022stochastic} is as follows.
We maintain a set of $m$ models $\theta_{i}=[\theta^{k}_{i}]_{k=1}^{K}$, where $i\in[m]$, and $\theta^{k}_{i}=[\textbf{x}_{i},\textbf{z}_{i}^{k}]$. At each iteration, given the non-shared parts $\textbf{z}_{i}^{k}$ for $i\in[m],k\in[K]$, we sample the shared part from the multiple target distributions $p(\textbf{x}|\textbf{z}^{k}, \mathbb{D})$, $k\in [K]$. Here we apply MWGraD to sample the shared parts $[\textbf{x}_{i}]_{i=1}^{m}$ from the multiple target distributions. Then, given the shared parts $[\textbf{x}_{i}]_{i=1}^{m}$, for each task $k$, we update the corresponding non-shared parts $[\textbf{z}_{i}^{k}]_{i=1}^{m}$ by sampling from the posterior distribution $p(\textbf{z}^{k}|\textbf{x}, \mathbb{D})$. This process corresponds to Bayesian sampling, for which we can use techniques, including SVGD, Blob methods or a neural network, as explained in the previous section. In our experiments, we use SVGD to sample the non-shared parts, while exploring different methods for sampling the shared parts, including MOO-SVGD, MT-SGD, MWGraD-SVGD, MWGraD-Blob and MWGraD-NN.\\

\noindent
\textbf{Datasets and Evaluation Metric}. We validate the methods on three benchmark datasets: Multi-Fashion-MNIST \citep{sabour2017dynamic}, Multi-MNIST \citep{phan2022stochastic}, and Multi-Fashion \citep{phan2022stochastic}. Each of them consists of 120,000 training and 20,000 testing images from MNIST and FashionMNIST \citep{xiao2017fashion} by overlaying an image on top of another. We compare our methods with the following baselinese: MOO-SVGD, MT-SGVD and MGDA, which achieved the best performance \citep{phan2022stochastic}. For the variants of MWGraD and MT-SGD, the reported results are from the ensemble prediction of five particle models. For MGDA, we train five particle models independently with different initializations and then ensemble these models. For evaluation metric, we compare the methods in terms of the accuracy for each task.\\

\noindent
\textbf{Results}. Table \ref{tab:results} shows the ensemble accuracy of compared methods across three datasets. We observe that variants of MWGraD consistently outperform the other methods. For example, MWGraD-Blob achieves the best performance of 96.71\%, 97.64\% for Task 1 on the Multi-Fashion+MNIST and Multi-MNIST, resepectively, while MWGraD-NN achieves the best performance of 87.2\% on the third dataset. For Task 2, MWGraD-Blob again achieves the best performance across all datasets, with average accuracies of 92.49\%, 96.69\% and 87.2\%, respectively. Additionally, MWGraD-SVGD shows comparable performance to MT-SGD on all datasets. The reason is that both methods use the same velocity approximation, with the main difference being in the update of the weights $\textbf{w}$: while MT-SGD solves for the optimal solution of $\textbf{w}$ in Theorem \ref{theorem1}, MWGraD-SVGD approximates the optimal solution by performing a gradient update (\ref{eqn:updatew}). These experimental results clearly demonstrate the effectiveness of our proposed methods for the application of multi-task learning. See more details on the experiments in Appendix \ref{exp:realworlddatasets}.

\section{Conclusion}
In this paper, we have addressed the  MODO problem, where the goal is to simultaneously minimize multiple functionals of probability distributions. We have introduced MWGraD, an iterative particle-based algorithm for solving MODO. At each iteration, it estimates the Wasserstein gradient for each objective using SGVD, Blob methods and neural networks. It then aggregates these gradients into one single Wasserstein gradient, which guides the updates of each particle. We have provided theoretical analyses and presented experiments on both synthetic and real-world datasets, demonstrating the effectiveness of MWGraD in identifying the joint high-density regions of objectives. In future work, we plan to explore MWGraD for real-world applications.

\subsubsection*{Acknowledgements}
This work was supported in part by the International Collaborative Research Program of Institute for Chemical Research, Kyoto University (grant \#2025-29), MEXT KAKENHI [grant number: 23K16939] (to D.H.N.), MEXT KAKENHI [grant numbers: 21H05027, 22H03645, 25H01144] (to H.M.), and JSPS KAKENHI [grant number: JP24H00685] (to A. N.).

\bibliographystyle{plainnat}
\bibliography{mlj}  

\begin{thebibliography}{33}
\providecommand{\natexlab}[1]{#1}
\providecommand{\url}[1]{\texttt{#1}}
\expandafter\ifx\csname urlstyle\endcsname\relax
  \providecommand{\doi}[1]{doi: #1}\else
  \providecommand{\doi}{doi: \begingroup \urlstyle{rm}\Url}\fi

\bibitem[Carrillo et~al.(2019)Carrillo, Craig, and
  Patacchini]{carrillo2019blob}
Jos{\'e}~Antonio Carrillo, Katy Craig, and Francesco~S Patacchini.
\newblock A blob method for diffusion.
\newblock \emph{Calculus of Variations and Partial Differential Equations},
  58:\penalty0 1--53, 2019.

\bibitem[Chen and Zhang(2017)]{chen2017particle}
Changyou Chen and Ruiyi Zhang.
\newblock Particle optimization in stochastic gradient mcmc.
\newblock \emph{arXiv preprint arXiv:1711.10927}, 2017.

\bibitem[Chen et~al.(2018)Chen, Zhang, Wang, Li, and Chen]{chen2018unified}
Changyou Chen, Ruiyi Zhang, Wenlin Wang, Bai Li, and Liqun Chen.
\newblock A unified particle-optimization framework for scalable bayesian
  sampling.
\newblock \emph{arXiv preprint arXiv:1805.11659}, 2018.

\bibitem[Chen et~al.(2020)Chen, Ngiam, Huang, Luong, Kretzschmar, Chai, and
  Anguelov]{chen2020just}
Zhao Chen, Jiquan Ngiam, Yanping Huang, Thang Luong, Henrik Kretzschmar, Yuning
  Chai, and Dragomir Anguelov.
\newblock Just pick a sign: Optimizing deep multitask models with gradient sign
  dropout.
\newblock \emph{Advances in Neural Information Processing Systems},
  33:\penalty0 2039--2050, 2020.

\bibitem[Cust{\'o}dio et~al.(2011)Cust{\'o}dio, Madeira, Vaz, and
  Vicente]{custodio2011direct}
Ana~Lu{\'\i}sa Cust{\'o}dio, JF~Aguilar Madeira, A~Ismael~F Vaz, and
  Lu{\'\i}s~Nunes Vicente.
\newblock Direct multisearch for multiobjective optimization.
\newblock \emph{SIAM Journal on Optimization}, 21\penalty0 (3):\penalty0
  1109--1140, 2011.

\bibitem[Deb et~al.(2016)Deb, Sindhya, and Hakanen]{deb2016multi}
Kalyanmoy Deb, Karthik Sindhya, and Jussi Hakanen.
\newblock Multi-objective optimization.
\newblock In \emph{Decision sciences}, pages 161--200. CRC Press, 2016.

\bibitem[D{\'e}sid{\'e}ri(2012)]{desideri2012multiple}
Jean-Antoine D{\'e}sid{\'e}ri.
\newblock Multiple-gradient descent algorithm (mgda) for multiobjective
  optimization.
\newblock \emph{Comptes Rendus Mathematique}, 350\penalty0 (5-6):\penalty0
  313--318, 2012.

\bibitem[Goodfellow et~al.(2020)Goodfellow, Pouget-Abadie, Mirza, Xu,
  Warde-Farley, Ozair, Courville, and Bengio]{goodfellow2020generative}
Ian Goodfellow, Jean Pouget-Abadie, Mehdi Mirza, Bing Xu, David Warde-Farley,
  Sherjil Ozair, Aaron Courville, and Yoshua Bengio.
\newblock Generative adversarial networks.
\newblock \emph{Communications of the ACM}, 63\penalty0 (11):\penalty0
  139--144, 2020.

\bibitem[Jordan et~al.(1998)Jordan, Kinderlehrer, and
  Otto]{jordan1998variational}
Richard Jordan, David Kinderlehrer, and Felix Otto.
\newblock The variational formulation of the fokker--planck equation.
\newblock \emph{SIAM journal on mathematical analysis}, 29\penalty0
  (1):\penalty0 1--17, 1998.

\bibitem[Kingma and Welling(2014)]{kingma2014stochastic}
Diederik~P Kingma and Max Welling.
\newblock Stochastic gradient vb and the variational auto-encoder.
\newblock In \emph{Second international conference on learning representations,
  ICLR}, volume~19, page 121, 2014.

\bibitem[Liu et~al.(2021{\natexlab{a}})Liu, Liu, Jin, Stone, and
  Liu]{liu2021conflict}
Bo~Liu, Xingchao Liu, Xiaojie Jin, Peter Stone, and Qiang Liu.
\newblock Conflict-averse gradient descent for multi-task learning.
\newblock \emph{Advances in Neural Information Processing Systems},
  34:\penalty0 18878--18890, 2021{\natexlab{a}}.

\bibitem[Liu et~al.(2021{\natexlab{b}})Liu, Zhang, Yang, Babanezhad, and
  Wang]{liu2021infinite}
Lewis Liu, Yufeng Zhang, Zhuoran Yang, Reza Babanezhad, and Zhaoran Wang.
\newblock Infinite-dimensional optimization for zero-sum games via variational
  transport.
\newblock In \emph{International Conference on Machine Learning}, pages
  7033--7044. PMLR, 2021{\natexlab{b}}.

\bibitem[Liu and Wang(2016)]{liu2016stein}
Qiang Liu and Dilin Wang.
\newblock Stein variational gradient descent: A general purpose bayesian
  inference algorithm.
\newblock \emph{Advances in neural information processing systems}, 29, 2016.

\bibitem[Liu et~al.(2021{\natexlab{c}})Liu, Tong, and Liu]{liu2021profiling}
Xingchao Liu, Xin Tong, and Qiang Liu.
\newblock Profiling pareto front with multi-objective stein variational
  gradient descent.
\newblock \emph{Advances in Neural Information Processing Systems},
  34:\penalty0 14721--14733, 2021{\natexlab{c}}.

\bibitem[Ma et~al.(2018)Ma, Zhao, Yi, Chen, Hong, and Chi]{ma2018modeling}
Jiaqi Ma, Zhe Zhao, Xinyang Yi, Jilin Chen, Lichan Hong, and Ed~H Chi.
\newblock Modeling task relationships in multi-task learning with multi-gate
  mixture-of-experts.
\newblock In \emph{Proceedings of the 24th ACM SIGKDD international conference
  on knowledge discovery \& data mining}, pages 1930--1939, 2018.

\bibitem[Nguyen and Sakurai(2023)]{nguyen2023mirror}
Dai~Hai Nguyen and Tetsuya Sakurai.
\newblock Mirror variational transport: a particle-based algorithm for
  distributional optimization on constrained domains.
\newblock \emph{Machine Learning}, pages 1--25, 2023.

\bibitem[Nguyen and Sakurai(2024)]{nguyen2024moreau}
Dai~Hai Nguyen and Tetsuya Sakurai.
\newblock Moreau-yoshida variational transport: a general framework for solving
  regularized distributional optimization problems.
\newblock \emph{Machine Learning}, 113\penalty0 (9):\penalty0 6697--6724, 2024.

\bibitem[Nguyen and Tsuda(2023)]{nguyen2023linear}
Dai~Hai Nguyen and Koji Tsuda.
\newblock On a linear fused gromov-wasserstein distance for graph structured
  data.
\newblock \emph{Pattern Recognition}, page 109351, 2023.

\bibitem[Nguyen et~al.(2021)Nguyen, Nguyen, and Mamitsuka]{nguyen2021learning}
Dai~Hai Nguyen, Canh~Hao Nguyen, and Hiroshi Mamitsuka.
\newblock Learning subtree pattern importance for weisfeiler-lehman based graph
  kernels.
\newblock \emph{Machine Learning}, 110:\penalty0 1585--1607, 2021.

\bibitem[Nguyen et~al.(2023)Nguyen, Sakurai, and
  Mamitsuka]{nguyen2023wasserstein}
Dai~Hai Nguyen, Tetsuya Sakurai, and Hiroshi Mamitsuka.
\newblock Wasserstein gradient flow over variational parameter space for
  variational inference.
\newblock \emph{arXiv preprint arXiv:2310.16705}, 2023.

\bibitem[Petric~Maretic et~al.(2019)Petric~Maretic, El~Gheche, Chierchia, and
  Frossard]{petric2019got}
Hermina Petric~Maretic, Mireille El~Gheche, Giovanni Chierchia, and Pascal
  Frossard.
\newblock Got: an optimal transport framework for graph comparison.
\newblock \emph{Advances in Neural Information Processing Systems}, 32, 2019.

\bibitem[Phan et~al.(2022)Phan, Tran, Le, Tran, Ho, and
  Phung]{phan2022stochastic}
Hoang Phan, Ngoc Tran, Trung Le, Toan Tran, Nhat Ho, and Dinh Phung.
\newblock Stochastic multiple target sampling gradient descent.
\newblock \emph{Advances in neural information processing systems},
  35:\penalty0 22643--22655, 2022.

\bibitem[Rudin(2021)]{rudin2021principles}
Walter Rudin.
\newblock Principles of mathematical analysis.
\newblock 2021.

\bibitem[Sabour et~al.(2017)Sabour, Frosst, and Hinton]{sabour2017dynamic}
Sara Sabour, Nicholas Frosst, and Geoffrey~E Hinton.
\newblock Dynamic routing between capsules.
\newblock \emph{Advances in neural information processing systems}, 30, 2017.

\bibitem[Santambrogio(2015)]{santambrogio2015optimal}
Filippo Santambrogio.
\newblock Optimal transport for applied mathematicians.
\newblock \emph{Birk{\"a}user, NY}, 55\penalty0 (58-63):\penalty0 94, 2015.

\bibitem[Thomas et~al.(2021)Thomas, Pineau, Laroche, et~al.]{thomas2021multi}
Philip~S Thomas, Joelle Pineau, Romain Laroche, et~al.
\newblock Multi-objective spibb: Seldonian offline policy improvement with
  safety constraints in finite mdps.
\newblock \emph{Advances in Neural Information Processing Systems},
  34:\penalty0 2004--2017, 2021.

\bibitem[Villani(2021)]{villani2021topics}
C{\'e}dric Villani.
\newblock \emph{Topics in optimal transportation}, volume~58.
\newblock American Mathematical Soc., 2021.

\bibitem[Villani et~al.(2009)]{villani2009optimal}
C{\'e}dric Villani et~al.
\newblock \emph{Optimal transport: old and new}, volume 338.
\newblock Springer, 2009.

\bibitem[Welling and Teh(2011)]{welling2011bayesian}
Max Welling and Yee~W Teh.
\newblock Bayesian learning via stochastic gradient langevin dynamics.
\newblock In \emph{Proceedings of the 28th international conference on machine
  learning (ICML-11)}, pages 681--688. Citeseer, 2011.

\bibitem[Xiao et~al.(2017)Xiao, Rasul, and Vollgraf]{xiao2017fashion}
Han Xiao, Kashif Rasul, and Roland Vollgraf.
\newblock Fashion-mnist: a novel image dataset for benchmarking machine
  learning algorithms.
\newblock \emph{arXiv preprint arXiv:1708.07747}, 2017.

\bibitem[Yu et~al.(2020)Yu, Kumar, Gupta, Levine, Hausman, and
  Finn]{yu2020gradient}
Tianhe Yu, Saurabh Kumar, Abhishek Gupta, Sergey Levine, Karol Hausman, and
  Chelsea Finn.
\newblock Gradient surgery for multi-task learning.
\newblock \emph{Advances in Neural Information Processing Systems},
  33:\penalty0 5824--5836, 2020.

\bibitem[Zhang and Sra(2016)]{zhang2016first}
Hongyi Zhang and Suvrit Sra.
\newblock First-order methods for geodesically convex optimization.
\newblock In \emph{Conference on learning theory}, pages 1617--1638. PMLR,
  2016.

\bibitem[Zhang et~al.(2024)Zhang, Xiao, Ji, and Zou]{zhang2024convergence}
Qi~Zhang, Peiyao Xiao, Kaiyi Ji, and Shaofeng Zou.
\newblock On the convergence of multi-objective optimization under generalized
  smoothness.
\newblock \emph{arXiv preprint arXiv:2405.19440}, 2024.

\end{thebibliography}






\onecolumn
\appendix
\title{Wasserstein Gradient Flow over Variational Parameter Space for Variational
Inference\\(Supplementary Material)}

\section{Proof of Theorem \ref{theorem1}}
\label{appendix:proofoftheorem1}
\begin{proof}
Problem (\ref{eqn:3}) can be equivalently rewritten as the following optimization problem
\begin{align}
\label{reformulatedproblem}
\begin{split}
    \min_{\textbf{v},\mu} - \mu + \frac{1}{2}\int_{\mathcal{X}}\langle \textbf{v}(\textbf{x}), \textbf{v}(\textbf{x})\rangle q^{(t)}(\textbf{x})d\textbf{x}\\
    \text{such that } \mu \leq \int_{\mathcal{X}}\langle \textbf{v}^{(t)}_{k}(\textbf{x}),\textbf{v}(\textbf{x}) \rangle q^{(t)}(\textbf{x})d\textbf{x}, \text{ for } k\in[K].
\end{split}
\end{align}
The Lagrange function is defined as follows
\begin{align*}
    \mathcal{L}(\textbf{v}, \mu, \textbf{w}) = - \mu + \frac{1}{2}\int_{\mathcal{X}}\langle \textbf{v}(\textbf{x}), \textbf{v}(\textbf{x})\rangle q^{(t)}(\textbf{x})d\textbf{x} + \sum_{k=1}^{K}w_{k} \left(\mu- \int_{\mathcal{X}}\langle \textbf{v}^{(t)}_{k}(\textbf{x}),\textbf{v}(\textbf{x}) \rangle q^{(t)}(\textbf{x})d\textbf{x}  \right),
\end{align*}
where $\textbf{w}=[w_1, w_2,...,w_K]^\top\in\mathcal{W}$ are the dual variables corresponding to $K$ constraints in (\ref{reformulatedproblem}). Taking the derivative of $\mathcal{L}$ with respect to $\textbf{v}$ and setting it to zero, we obtain the optimal solution $\textbf{v}^{*}$ as follows. For $\textbf{x}\in\mathcal{X}$, we have that
\begin{align}
\label{solutionx}
    \textbf{v}^{*}(\textbf{x}) = \sum_{k=1}^{K}w_{k}\textbf{v}^{(t)}_{k}(\textbf{x}).
\end{align}
Substituting it to (\ref{reformulatedproblem}), we have the following optimization problem for $\textbf{w}$
\begin{equation*}
    \label{problem:solveweights}
        \textbf{w}^{*} = \argmin_{\textbf{w}\in\mathcal{W}}\frac{1}{2}\int_{\mathcal{X}}\lVert \sum_{k=1}^{K}w_{k}\textbf{v}^{(t)}_{k}(\textbf{x})  \rVert^{2}_{2}q^{(t)}(\textbf{x})d\textbf{x},
\end{equation*}
which completes the proof.
\end{proof}

\section{Algorithm}
\label{appendix:algorithm}
\begin{algorithm}[H]
\label{alg:mwgrad}
\SetAlgoLined
\KwIn{Functionals $\left\{F_{k}\right\}$, number of particles $m$, number of iterations $T$, step sizes $\alpha>0, \beta > 0$, weights $\textbf{w}^{(0)}$.}
\KwOut{a set of $m$ particles $\left\{\textbf{x}^{(T)}_{i} \right\}_{i=1}^{m}$.}
Sample $m$ initial particles $\left\{\textbf{x}^{(0)}_{i} \right\}_{i=1}^{m}$ from $\mathcal{N}(0, \textbf{I}_{d})$.\\
$t\leftarrow 0$\\
\While{$t< T$}{
     Estimate $\Tilde{\textbf{v}}^{(t)}_{k}(\textbf{x}^{(t)}_{i})$ by (\ref{approximate-svgd}) or (\ref{approximate-blob}) or (\ref{approximate-nn}), for $i\in[m]$, $k\in [K]$

     Compute $\Tilde{\textbf{v}}^{(t)}(\textbf{x}^{(t)}_{i})\leftarrow\sum_{k=1}^{K}w^{(t)}_{k}\Tilde{\textbf{v}}^{(t)}_{k}(\textbf{x}^{(t)}_{i})$, for $i\in [m]$

     Update $\textbf{x}^{(t+1)}_{i}\leftarrow \textbf{x}^{(t)}_{i} - \alpha \Tilde{\textbf{v}}^{(t)}(\textbf{x}^{(t)}_{i})$, for $i\in[m]$

     Update $\textbf{w}^{(t)}$ by (\ref{eqn:updatew})
     
    $t\leftarrow t+1$
    }
 \caption{Multi-objective Wasserstein Gradient Descent (MWGraD)}
\end{algorithm}

\section{A Detailed Discussion on the Pareto Stationary Distribution}
\label{appendix:pareto}
In this section, we examine the conditions under which a Pareto stationary distribution is also a Pareto optimal distribution. We begin with the following claim.

\textbf{Claim 1} (Pareto optimality $\rightarrow$ Pareto stationarity).  
If a distribution $q$ is not Pareto stationary, then there exists a descent direction that simultaneously improves all objective functions.

\textit{Proof}.  
Suppose that at a point $q \in \mathcal{P}_2(\mathcal{X})$, there exists no convex combination of the Wasserstein gradients that sums to zero. That is, for all $\mathbf{w} \in \mathbb{R}^K$ with $\sum_{k=1}^{K} w_k = 1$, we have
\[
\texttt{grad} \mathbf{F}(q)\, \mathbf{w} = \sum_{k=1}^{K} w_k \, \texttt{grad} F_k(q) \neq 0.
\]
This implies that the gradients $\texttt{grad} F_k(q)$ for $k \in [K]$ all lie within a common open half-space. Therefore, there exists a tangent vector $s$ lying in the opposite half-space such that
\[
\langle s, \texttt{grad} F_k(q) \rangle_{q} < 0 \quad \text{for all } k \in [K].
\]
This means that moving $q$ along the direction $s$ results in a simultaneous decrease of all objectives $F_k$. Hence, $q$ cannot be a Pareto optimal distribution.

From the above argument, we conclude that Pareto optimality implies Pareto stationarity; in other words, Pareto optimality is a stronger condition.

\textbf{Claim 2} (Pareto stationarity + geodesic strict convexity $\rightarrow$ Pareto optimality).  
If all objectives are (geodesically) strictly convex, and $q$ is a Pareto stationary distribution, then $q$ is a Pareto optimal distribution.

\textit{Proof}.  
Assume $q$ is a Pareto stationary distribution. Now suppose, for contradiction, that $q$ is not Pareto optimal. Then there exists another distribution $q' \neq q$ such that
\[
F_k(q') \leq F_k(q) \quad \text{for all } k \in [K], \quad \text{and} \quad F_j(q') < F_j(q) \text{ for some } j.
\]
Let $\gamma:[0,1]\rightarrow \mathcal{P}_2(\mathcal{X})$ be a curve that connects $q$ and $q'$ and satisfies $\gamma(0)=q$ and $\gamma(1)=q'$. Let $s$ be a tangent vector at $q$ and $s$ satisfies $s=\gamma'(0)$. As $q$ is a Pareto stationary distribution, for any direction $s$ in the tangent space at $q$, there is at least one objective functional $F_k$ such that
\[
\langle s, \texttt{grad} F_k(q) \rangle_q \geq 0.
\]

Since $F_k$ is strictly convex along the curve $\gamma$, we have
\[
F_k(q') > F_k(q) + \langle \texttt{grad} F_k(q), s \rangle_q.
\]
But $F_k(q') \leq F_k(q)$, which implies $\langle \texttt{grad} F_k(q), s \rangle_q < 0$. This contradicts the stationarity of $q$. Therefore, $q$ must be Pareto optimal.

\section{Formal version of Theorem \ref{theorem:convergence_simple} and its proof}
\label{appendix:proofoftheorem2}
We first present the following lemma, which is useful to prove Theorem \ref{theorem:convergence_simple}.
\begin{lem}
\label{lemma1}
    Assume that a functional $G$ is geodesically $L$-smooth and $G(q) - G^{*}\leq C$ for $q\in \mathcal{P}_{2}(\mathcal{X})$, where $G^{*}=\argmin_{p\in \mathcal{P}_{2}(\mathcal{X})} G(p)$, and $C$ is a constant. Then, there exists a constant $M$ such that
    \begin{equation}
        \langle \texttt{grad} G(q), \texttt{grad} G(q)\rangle_{q}\leq M^{2}
    \end{equation}
\end{lem}
\begin{proof}
    Since $G$ is geodesically $L$-smooth, we have for both $q,q^\prime\in \mathcal{P}_{2}(\mathcal{X})$ that
    \begin{align}
        G(q^\prime) \leq G(q) + \langle \texttt{grad} G(q), \texttt{Exp}^{-1}_{q}(q^\prime)\rangle_{q} + \frac{L}{2}\langle\texttt{Exp}^{-1}_{q}(q^\prime), \texttt{Exp}^{-1}_{q}(q^\prime)\rangle_{q}
    \end{align}
    By choosing $q^\prime = \texttt{Exp}_{q}(-\frac{1}{L} \texttt{grad} G(q))$, we have
    \begin{align}
        G(q^\prime) \leq G(q) - \frac{1}{L} \langle\texttt{grad} G(q), \texttt{grad} G(q)\rangle_{q} + \frac{1}{2L}\langle\texttt{grad} G(q), \texttt{grad} G(q)\rangle_{q}
    \end{align}
    So, we have
    \begin{align}
        \langle\texttt{grad} G(q), \texttt{grad} G(q)\rangle_{q}\leq 2L \left( G(q)- G(q^\prime) \right) \leq 2L \left( G(q)- G^{*} \right)\leq 2LC
    \end{align}
    By setting $M^{2}=2LC$, we conclude that
    \begin{equation}
        \langle\texttt{grad} G(q), \texttt{grad} G(q)\rangle_{q} \leq M^{2}
    \end{equation}
\end{proof}
\noindent
Let $b_{1}>0$, $b_{2}>0$, $b_{3}>0$ and $C>0$ be some constants such that
\begin{align*}
    \Delta + b_{1} + b_{2} + b_{3} \leq C.
\end{align*}
Define $L = \max_{k\in[K]}\left\{\ell_{k} \right\}$ and $M= 2LC$. Then we have the following convergence theorem for Algorithm \ref{alg:mwgrad}.

\begin{thm}
\label{theorem:convergence_detailed}
    Let Assumptions 1 and 2 hold. Set $\alpha, \beta, T$ as follows:
    \begin{align*}
        \beta &\leq \frac{\epsilon^{2}}{24(K M^{4} + 2 K M^{2}\sigma^{2} + K \sigma^{4})}\\
        \alpha &\leq \min \left\{\frac{4 b_{1}}{3 T \sigma^{2}}, \beta b_{2}, \frac{b_{3}}{2\beta T (K M^{4} + 2 K M \sigma^{2} + K \sigma^{4})}\right\}\\
        T&\geq \max \left\{ \frac{12\Delta}{\alpha \epsilon^{2}}, \frac{12}{\beta \epsilon^{2}} \right\}.
    \end{align*}
    We then have that
    \begin{align}
    \label{eqn:graddientnorm}
        \frac{1}{T}\sum_{t=0}^{T-1}\langle \texttt{grad} \textbf{F}(q^{(t)})\textbf{w}^{(t)}, \texttt{grad} \textbf{F}(q^{(t)})\textbf{w}^{(t)} \rangle_{q^{(t)}} \leq \epsilon^{2} + 3\sigma^{2}.
    \end{align}
\end{thm}

\begin{proof}
As shown in Lemma \ref{lemma1}, a bounded functional value implies a bounded Wasserstein gradient norm. In the following, we show that, with the properly selected parameters in Theorem \ref{theorem:convergence_detailed}, the functional value is bounded by induction. 

For the base case, it is trivial to see that $F_{k}(q^{(0)}) - F_{k}^{*}\leq \Delta \leq C$ for $k\in [K]$. Now we assume that for any $k\in[K]$ and $t\leq \tau< T$, we have that $F_{k}(q^{(t)}) - F_{k}^{*} \leq C$ holds. We then prove that $F_{k}(q^{(\tau+1)}) - F_{k}^{*}\leq  C$ holds for any $k\in[K]$.

For any $k\in [K]$, $t\leq \tau$, we have $F_{k}(q^{(t)})-F_{k}^{*}\leq C$, which implies that $\langle \texttt{grad} F_{k}(q^{(t)}), \texttt{grad} F_{k}(q^{(t)})\rangle_{q^{(t)}} \leq M$.
Since $F_{k}$ is $\ell_{k}$-smooth, it follows that
\begin{align*}
    F_{k}(q^{(\tau+1)}) \leq F_{k}(q^{(\tau)}) + \langle \texttt{grad} F_{k}(q^{(\tau)}), \texttt{Exp}_{q^{(\tau)}}^{-1}(q^{(\tau+1)})\rangle + \frac{\ell_{k}}{2}\langle\texttt{Exp}_{q^{(\tau)}}^{-1}(q^{(\tau+1)}), \texttt{Exp}_{q^{(\tau)}}^{-1}(q^{(\tau+1)})\rangle_{q^{(\tau)}}.
\end{align*}
Since $\texttt{Exp}_{q^{(t)}}^{-1}(q^{(t+1)}) = (\texttt{grad} \textbf{F}(q^{(t)}) + \xi^{(t)})\textbf{w}^{(t)}$, and $\ell_{k} \leq \max_{i\in [K]}\ell_{i}=L$, it follows that
\begin{align}
\begin{split}
    F_{k}(q^{(\tau+1)}) &\leq F_{k}(q^{(\tau)}) - \alpha \langle \texttt{grad} F_{k}(q^{(\tau)}), (\texttt{grad} \textbf{F}(q^{(\tau)}) + \xi^{(\tau)})\textbf{w}^{(\tau)}\rangle_{q^{(\tau)}} \\
    &+ \frac{L\alpha^{2}}{2}\langle (\texttt{grad} \textbf{F}(q^{(\tau)}) + \xi^{(\tau)})\textbf{w}^{(\tau)}, (\texttt{grad} \textbf{F}(q^{(t)}) + \xi^{(\tau)})\textbf{w}^{(\tau)}\rangle_{q^{(\tau)}}.
\end{split}
\end{align}
\noindent
For any $\textbf{w}\in \mathcal{W}$, we have that
\begin{align}
\label{appendix:inequality-1}
\begin{split}
    \textbf{F}(q^{(\tau+1)})\textbf{w} &\leq \textbf{F}(q^{(\tau)})\textbf{w} - \alpha \langle \texttt{grad} F(q^{(\tau)})\textbf{w}, (\texttt{grad} \textbf{F}(q^{(\tau)}) + \xi^{(\tau)})\textbf{w}^{(t)}\rangle_{q^{(\tau)}} \\
    &+ \frac{L \alpha^{2}}{2} \langle (\texttt{grad} F(q^{(\tau)}) + \xi^{(\tau)})\textbf{w}^{(\tau)}, (\texttt{grad} F(q^{(\tau)}) + \xi^{(\tau)})\textbf{w}^{(\tau)}\rangle_{q^{(\tau)}}\\
    &\leq \textbf{F}(q^{(\tau)})\textbf{w} - \alpha \langle \texttt{grad} \textbf{F}(q^{(\tau)})\textbf{w}^{(\tau)}, \texttt{grad} \textbf{F}(q^{(\tau)})\textbf{w}^{(\tau)}\rangle_{q^{(t)}} - \alpha \langle \texttt{grad} \textbf{F}(q^{(\tau)})\textbf{w}^{(\tau)},\xi^{(t)}\textbf{w}^{(\tau)}\rangle_{q^{(\tau)}}\\
    &+ \alpha \langle \texttt{grad} \textbf{F}(p^{(\tau)})(\textbf{w}^{(\tau)}-\textbf{w}), (\texttt{grad} \textbf{F}(q^{(\tau)}) + \xi^{(\tau)})\textbf{w}^{(\tau)}\rangle_{q^{(\tau)}}\\
    &+\frac{L}{2}\alpha^{2} \langle(\texttt{grad} \textbf{F}(p^{(\tau)}) + \xi^{(\tau)})\textbf{w}^{(\tau)}, (\texttt{grad} \textbf{F}(q^{(\tau)}) + \xi^{(\tau)})\textbf{w}^{(\tau)}\rangle_{q^{(\tau)}}.
\end{split}
\end{align}
By the basic inequality $(a+b)^{2} \leq 2(a^{2}+b^{2})$, we have that
\begin{align}
\label{appendix:inequality-2}
\begin{split}
    \langle(\texttt{grad} \textbf{F}(q^{(\tau)}) + \xi^{(\tau)})\textbf{w}^{(\tau)}, (\texttt{grad} \textbf{F}(q^{(\tau)}) + \xi^{(\tau)})\textbf{w}^{(\tau)}\rangle_{q^{(\tau)}} \leq 2 \langle \texttt{grad} \textbf{F}(q^{(\tau)})\textbf{w}^{(\tau)}, \texttt{grad} \textbf{F}(q^{(\tau)})\textbf{w}^{(\tau)}\rangle_{q^{(\tau)}} \\
    + 2 \langle\xi^{(t)}\textbf{w}^{(\tau)}, \xi^{(\tau)}\textbf{w}^{(\tau)}\rangle_{q^{(\tau)}},\\
    \langle \texttt{grad} \textbf{F}(q^{\tau})\textbf{w}^{(\tau)}, \xi^{(\tau)}\textbf{w}^{(\tau)} \rangle  \leq \frac{1}{2} \langle\texttt{grad} \textbf{F}(q^{(\tau)})\textbf{w}^{(\tau)}, \texttt{grad} \textbf{F}(q^{(\tau)})\textbf{w}^{(\tau)}\rangle_{q^{(\tau)}} 
    + \frac{1}{2}  \langle\xi^{(t)}\textbf{w}^{(\tau)}, \xi^{(\tau)}\textbf{w}^{(\tau)}\rangle_{q^{(\tau)}}.
\end{split}
\end{align}
Thus, combining (\ref{appendix:inequality-1}), (\ref{appendix:inequality-2}), we have that
\begin{align}
\label{appendix:inequality-3}
\begin{split}
    \textbf{F}(q^{(\tau+1)})\textbf{w} &\leq \textbf{F}(q^{(\tau)})\textbf{w} - \alpha (\frac{1}{2}- L\alpha) \langle \texttt{grad} \textbf{F}(q^{(\tau)})\textbf{w}^{(\tau)}, \texttt{grad} \textbf{F}(q^{(\tau)})\textbf{w}^{(\tau)}\rangle_{q^{(\tau)}}\\
    & + \alpha\left(L\alpha +\frac{1}{2}\right)\langle \xi^{(\tau)}\textbf{w}^{(\tau)}, \xi^{(\tau)}\textbf{w}^{(\tau)}\rangle_{q^{(\tau)}}\\
    &+ \alpha \langle \texttt{grad} \textbf{F}(q^{(\tau)})(\textbf{w}^{(\tau)}-\textbf{w}), (\texttt{grad} \textbf{F}(q^{(\tau)}) + \xi^{(\tau)})\textbf{w}^{(\tau)}\rangle_{q^{(\tau)}}.
\end{split}
\end{align}
\noindent
Based on the update of $\textbf{w}^{(\tau)}$, we have that
\begin{align}
\begin{split}
    \textbf{w}^{(\tau+1)}=\Pi_{\mathcal{W}}\left(\textbf{w}^{(\tau)} - \beta  (\texttt{grad} \textbf{F}(q^{(\tau)}) + \xi^{(\tau)})^\top(\texttt{grad} \textbf{F}(q^{(\tau)}) + \xi^{(\tau)})\textbf{w}^{(\tau)} \right),
\end{split}
\end{align}
where we have used the following identity for notational simplicity
\begin{equation*}
    (\texttt{grad} \textbf{F}(q^{(\tau)}) + \xi^{(\tau)})^\top(\texttt{grad} \textbf{F}(q^{(\tau)}) + \xi^{(\tau)})=\int_{\mathcal{X}}\Tilde{\textbf{V}}^{(\tau)}(\textbf{x})^\top\Tilde{\textbf{V}}^{(\tau)}(\textbf{x})q^{(\tau)}(\textbf{x})d\textbf{x},
\end{equation*}
and $\Tilde{\textbf{V}}^{(\tau)}(\textbf{x})=\left[\Tilde{\textbf{v}}^{(\tau)}_{1}(\textbf{x}),\Tilde{\textbf{v}}^{(\tau)}_{2}(\textbf{x}),...,\Tilde{\textbf{v}}^{(\tau)}_{K}(\textbf{x})\right]$ is the approximation of $\textbf{V}^{(\tau)}(\textbf{x})$. 

\noindent
Applying the non-expansiveness of the projection, it follows that
\begin{align}
\begin{split}
    \lVert \textbf{w}^{(\tau+1)} - \textbf{w}\rVert^{2}_{2} &=\lVert \Pi_{\mathcal{W}}\left(\textbf{w}^{(\tau)} - \beta (\texttt{grad} \textbf{F}(q^{(\tau)}) + \xi^{(\tau)})^\top(\texttt{grad} \textbf{F}(q^{(\tau)}) + \xi^{(\tau)})\textbf{w}^{(\tau)} \right) - \textbf{w}\rVert^{2}_{2}\\
    &\leq  \lVert \textbf{w}^{(\tau)} - \textbf{w}\rVert^{2}_{2} + \beta^{2}\lVert  (\texttt{grad} \textbf{F}(q^{(\tau)}) + \xi^{(\tau)})^\top(\texttt{grad} \textbf{F}(q^{(\tau)}) + \xi^{(\tau)})\textbf{w}^{(\tau)} \rVert^{2}_{2}\\
    & - 2\beta \langle (\texttt{grad} \textbf{F}(q^{(\tau)}) + \xi^{(\tau)})(\textbf{w}^{(\tau)}-\textbf{w}), (\texttt{grad} \textbf{F}(q^{(\tau)}) + \xi^{(\tau)})\textbf{w}^{(\tau)}\rangle_{q^{(t)}}\\
    &= \lVert \textbf{w}^{(\tau)} - \textbf{w}\rVert^{2}_{2} + \beta^{2}\lVert (\texttt{grad} \textbf{F}(q^{(\tau)}) + \xi^{(\tau)})^\top(\texttt{grad} \textbf{F}(q^{(\tau)}) + \xi^{(\tau)})\textbf{w}^{(\tau)} \rVert^{2}_{2}\\
    &- 2\beta \langle \texttt{grad} \textbf{F}(q^{(\tau)})(\textbf{w}^{(\tau)}-\textbf{w}), (\texttt{grad} \textbf{F}(q^{(\tau)}) + \xi^{(\tau)})\textbf{w}^{(\tau)}\rangle_{q^{(t)}}\\
    &- 2\beta  \langle \xi^{(\tau)}(\textbf{w}^{(\tau)}-\textbf{w}), \texttt{grad} \textbf{F}(q^{(\tau)})\textbf{w}^{(\tau)} + \xi^{(\tau)} \textbf{w}^{(\tau)}\rangle_{q^{(t)}}.
\end{split}
\end{align}
Thus it follows that 
\begin{align}
\label{appendix:inequality-4}
\begin{split}
    \langle \texttt{grad} \textbf{F}(q^{(\tau)})(\textbf{w}^{(\tau)}-\textbf{w}) &, (\texttt{grad} \textbf{F}(q^{(\tau)}) + \xi^{(\tau)})\textbf{w}^{(\tau)}\rangle_{q^{(t)}} \leq \frac{1}{2\beta}\left(\lVert \textbf{w}^{(\tau)} - \textbf{w}\rVert^{2}_{2}-\lVert \textbf{w}^{(\tau+1)} - \textbf{w}\rVert^{2}_{2} \right)\\
    &- \langle \xi^{(t)} (\textbf{w}^{(\tau)}-\textbf{w}),  \texttt{grad} \textbf{F}(q^{(\tau)})w^{(\tau)} + \xi^{(\tau)}\textbf{w}^{(\tau)}\rangle_{q^{(t)}} \\
    &+ \frac{\beta}{2}\lVert (\texttt{grad} \textbf{F}(q^{(\tau)}) + \xi^{(\tau)})^\top(\texttt{grad} \textbf{F}(q^{(\tau)}) + \xi^{(\tau)})\textbf{w}^{(\tau)}\rVert^{2}_{q^{(t)}}\\
    & \leq  \frac{1}{2\beta}\left(\lVert \textbf{w}^{(\tau)} - \textbf{w}\rVert^{2}_{2}-\lVert \textbf{w}^{(\tau+1)} - \textbf{w}\rVert^{2} _{2}\right)  \\
    &- \langle \xi^{(t)}(\textbf{w}^{(\tau)}-\textbf{w}), \texttt{grad} \textbf{F}(q^{(\tau)})\textbf{w}^{(\tau)} + \delta^{(\tau)} \textbf{w}^{(\tau)}\rangle_{q^{(t)}}
    + 2\beta K M^{4} + 4\beta K M^{2}\sigma^{2} + 2\beta K \sigma^{4},
\end{split}
\end{align}
where the last inequality holds due to Assumption 2. Then plugging (\ref{appendix:inequality-4}) into (\ref{appendix:inequality-3}), we can show that
\begin{align}
\label{appendix:inequality-5}
\begin{split}
    \textbf{F}(q^{(\tau+1)})\textbf{w} - \textbf{F}(q^{(\tau)})\textbf{w} &\leq  - \alpha \left(\frac{1}{2}- L\alpha\right)  \langle \texttt{grad} \textbf{F}(q^{(\tau)})\textbf{w}^{(\tau)}, \texttt{grad} \textbf{F}(q^{(\tau)})\textbf{w}^{(\tau)} \rangle_{q^{(t)}} +  \alpha (\frac{1}{2}+L\alpha) \sigma^{2}\\
    & + \frac{\alpha}{2\beta}\left(\lVert \textbf{w}^{(\tau)} - \textbf{w}\rVert^{2}_{2}-\lVert \textbf{w}^{(\tau+1)} - \textbf{w}\rVert^{2}_{2} \right)\\
    & +\alpha  \sigma (M + \sigma) + 2\alpha\beta K M^{4} + 4\alpha\beta K M^{2}\sigma^{2} + 2\alpha\beta K\sigma^{4}.
\end{split}
\end{align}
Taking sum of (\ref{appendix:inequality-5}) from $t=0$ to $\tau$, for any $\textbf{w}\in \mathcal{W}$, we have that
\begin{align}
\label{appendix:inequality-6}
\begin{split}
    \textbf{F}(q^{(\tau+1)})\textbf{w} - \textbf{F}(q^{(0)})\textbf{w} &\leq 
    -\frac{\alpha}{4} \sum_{t=0}^{\tau}\langle \texttt{grad} \textbf{F}(q^{(t)})\textbf{w}^{(t)}, \texttt{grad} \textbf{F}(q^{(t)})\textbf{w}^{(t)}\rangle_{q^{(t)}}  + \frac{3}{4}\alpha T \sigma^{2}  \\
    & + \frac{\alpha}{2\beta}\left(\lVert \textbf{w}^{(0)}-\textbf{w} \rVert^{2}_{2}-\lVert \textbf{w}^{(\tau+1)}-\textbf{w} \rVert^{2}_{2} \right)\\
    & + 2\alpha \beta K M^{4}T + 4\alpha\beta K M^{2}\sigma^{2}T + 2\alpha\beta K \sigma^{4} T\\
    & \leq \frac{3}{4}\alpha T \sigma^{2} + \frac{\alpha}{\beta} + 2\alpha\beta T (K M^{4} + 2 K M^{2}\sigma^{2} + K \sigma^{4})
\end{split}
\end{align}
where the inequality is due to $\tau < T$, $\alpha L \leq 1/4$ and Assumption 2. Thus, we have that
\begin{align}
    \label{ineq:induction}
    \textbf{F}(q^{(\tau+1)})\textbf{w} - \textbf{F}^{*}\textbf{w} \leq  \underbrace{\textbf{F}(q^{(0)})\textbf{w} - \textbf{F}^{*}\textbf{w}}_{\leq \Delta} + \underbrace{\frac{3}{4}\alpha T \sigma^{2}}_{\leq b_{1}} + \underbrace{\frac{\alpha}{\beta}}_{\leq b_{2}} + \underbrace{2\alpha\beta T (K M^{4} + 2 K M^{2}\sigma^{2} + K \sigma^{4})}_{\leq b_{3}}\leq C.
\end{align}
Now we finish the induction step and can show that $F_{k}(q^{(\tau+1)})-F_{k}^{*} \leq C$ for all $k\in [K]$. Furthermore, according to (\ref{appendix:inequality-6}), for 
$\tau = T-1$, we have that
\begin{align}
\label{ineq:last}
\begin{split}
     \frac{1}{T}\sum_{t=0}^{T-1}\langle \texttt{grad} \textbf{F}(q^{(t)})\textbf{w}^{(t)}, \texttt{grad} \textbf{F}(q^{(t)})\textbf{w}^{(t)}\rangle_{q^{(t)}}  &\leq  \frac{4(\textbf{F}(q^{(0)})\textbf{w}-\textbf{F}^{*}\textbf{w})}{\alpha T} + \frac{4}{\beta T} + 8\beta (K M^{4} + 2 K M^{2}\sigma^{2} + K \sigma^{4}) + 3 \sigma^{2}\\
     &\leq  \epsilon^{2} + 3 \sigma^{2},
\end{split}
\end{align}
which completes the proof.
\end{proof}

\noindent
\textbf{Remarks}. Although there is a circular dependency among the parameters $\beta$, $\alpha$, and $T$, the required conditions in Theorem~\ref{theorem:convergence_detailed} can still be satisfied simultaneously. Specifically, if we choose 
\[
\alpha = \mathcal{O}(\epsilon^2), \quad \beta = \mathcal{O}(\epsilon^2), \quad \alpha T = \Theta(\epsilon^{-2}), \quad \beta T = \Theta(\epsilon^{-2}), \quad \text{and} \quad T = \Theta(\epsilon^{-4}),
\]
then all the necessary assumptions hold. To make this construction concrete, consider the following example.

Let us set 
\[
\beta = a_1 \epsilon^2, \quad \alpha = a_2 \epsilon^2, \quad \text{and} \quad a_4 \epsilon^{-4} \leq T \leq a_3 \epsilon^{-4},
\]
where the constants $a_1, a_2, a_3, a_4 > 0$ are chosen such that inequality~\eqref{ineq:induction} is satisfied. Specifically, we require
\begin{align*}
    0 &< a_1 \leq \frac{1}{24(K M^4 + 2 K M^2 \sigma^2 + K\sigma^4)}, \\
    0 &< a_2 \leq a_1 b_2, \\
    0 &< a_3 \leq \min\left( \frac{4 b_1 \epsilon^4}{3\sigma^2}, \frac{b_3 \epsilon^2}{2 a_1(K M^4 + 2 K M^2 \sigma^2 + K\sigma^4)} \right).
\end{align*}
With this choice of constants, it is straightforward to verify that inequality~\eqref{ineq:induction} holds.

Next, we choose $a_4$ such that $0 < a_4 \leq a_3$, ensuring that $T$ satisfies
\[
a_4 \epsilon^{-4} \leq T \leq a_3 \epsilon^{-4}.
\]
It follows that
\begin{align*}
    T\beta & \geq a_1 a_4 \epsilon^{-2}, \\
    T\alpha & \geq a_2 a_4 \epsilon^{-2}.
\end{align*}
We now analyze the upper bound in inequality~\eqref{ineq:last}:
\begin{align}
\begin{split}
     \frac{1}{T} \sum_{t=0}^{T-1} \left\langle \texttt{grad} \mathbf{F}(q^{(t)}) \mathbf{w}^{(t)}, \texttt{grad} \mathbf{F}(q^{(t)}) \mathbf{w}^{(t)} \right\rangle_{q^{(t)}} 
     &\leq \frac{4\Delta}{a_2 a_4} \epsilon^2 + \frac{4}{a_1 a_4} \epsilon^2 \\
     &\quad + 8(K M^4 + 2 K M^2 \sigma^2 + K\sigma^4)a_1\epsilon^2 + 3\sigma^2 \\
     &= \mathcal{O}(\epsilon^2) + 3\sigma^2.
\end{split}
\end{align}

Therefore, we have constructed a sequence of parameter choices $(a_1, a_2, a_3, a_4)$ such that all theoretical conditions are met. Moreover, the norm of the convex combination of Wasserstein gradients is bounded by a term of order $\mathcal{O}(\epsilon^2)$—reflecting the convergence rate of MWGraD with exact velocity fields—plus a $3\sigma^2$ term, which accounts for the error introduced by approximate gradient computations.

As established in Theorem \ref{theorem:convergence_detailed}, which analyzes the convergence of MWGraD, the squared norm of the convex combination of Wasserstein gradients is bounded by two terms: one of order $\epsilon^2$, and the other proportional to the approximation error $\sigma^2$. This result suggests that, in order to avoid convergence to a non–Pareto stationary point, it is important to minimize the approximation error $\sigma^2$ as much as possible. For instance, in the MWGraD-NN variant, where Wasserstein gradients are approximated using neural networks, we can reduce $\sigma^2$ by increasing the capacity of the networks—e.g., by adding more layers or neurons. However, this improvement in accuracy comes at the cost of increased computational complexity. It is also worth noting that Theorem \ref{theorem:convergence_detailed} does not rely on any geodesic convexity assumptions. Therefore, the convergence guarantees extend naturally to non-convex settings.

\section{Additional Experimental Results on synthetic datasets}
\label{exp:settings}
We consider $F_{k}(q)=D(q,\pi_{k})$, where $D$ is the KL divergence or JS divergence, and $\pi_{k}$ is characterized by a set of samples. We use the same target distribution used in the case of energy functional but we generate a set of samples to represent the target distribution instead of its energy functional. In our experiments, we generate 30 samples for each target distribution, so we have 120 samples for all objectives. For this case, MOO-SVGD, MWGraD-SVGD and MWGraD-Blob cannot be used. Thus we need to use the variational form of the KL and JS divergences. For KL divergence, we use its variational form (\ref{eqn:varformKL}). For JS divergence, as shown in \citep{nguyen2023mirror}, the variational form of the JS divergence is as follows:
\begin{align}
\label{eqn:varformJS}
    JS(q, \pi_{k}) = \sup_{h_{k}\in\mathcal{H}^{c}}{\left\{ \mathbb{E}_{\textbf{x}\sim q}\left[ h_{k}(\textbf{x}) \right]   -JS^{*}(h_{k})\right\}}
\end{align}
where $JS^{*}(h_{k})=-\frac{1}{2}\mathbb{E}_{\textbf{x}\sim\pi}\left[ \log\left( 1-2e^{2h_{k}(\textbf{x})}\right) \right]$, and $\mathcal{H}^{c}$ is the space of function $h$ that satisfies: $h(\textbf{x})< 1/2\log(1/2)$ for all $\textbf{x}\in \mathcal{X}$. We introduce the following change of variable: $h_{k}^\prime(\textbf{x})=1-2e^{2h_{k}(\textbf{x})}$. It is easy to verify that $0<h_{k}^\prime(\textbf{x})<1$ for all $\textbf{x}\in \mathcal{X}$. Then, the variational form of JS divergence can be rewritten as:
\begin{align*}
     \sup_{h_{k}^\prime\in\mathcal{H}^\prime}{\left\{ \mathbb{E}_{\textbf{x}\sim q}\left[ \log\left( 1-h_{k}^\prime(\textbf{x}) \right) \right]  + \mathbb{E}_{\textbf{y}\sim \pi}\left[ \log(h_{k}^\prime(\textbf{x})) \right]\right\}}
\end{align*}
where $\mathcal{H}^\prime$ is the space of functions whose outputs are in between 0 and 1. As we have access to samples drawn from $q$ and $\pi$, we can estimate $h_{k}^\prime$ and then estimate $h_{k}$ using: $h_{k}(\textbf{x})=\frac{1}{2}\log\left( \frac{1-h_{k}^\prime(\textbf{x})}{2} \right)$ for all $\textbf{x}\in\mathcal{X}$. For the function $h_{k}^\prime$, we also use a neural network with two layers, each of which has 50 neurons, to parameterize it. To guarantee the outputs of $h_{k}^\prime$ to be in $(0,1)$, we use the sigmoid activation function in the last layer. We set the same step sizes as in the previous experiments, i.e., $\alpha=0.0001, \beta=0.001$. We run 2000 updates of particles.

Figure \ref{fig:toy-2} shows the particles updated by MWGraD at selected iterations, for two cases of KL divergence and JS divergence. We observe that MWGraD tend to cover the common high-density regions with particles, which is consistent with our observation in cases where $F_{k}(q)$ is represented by an energy functional.

\begin{figure*}[t]
\centerline{\includegraphics[width=1.0\textwidth]{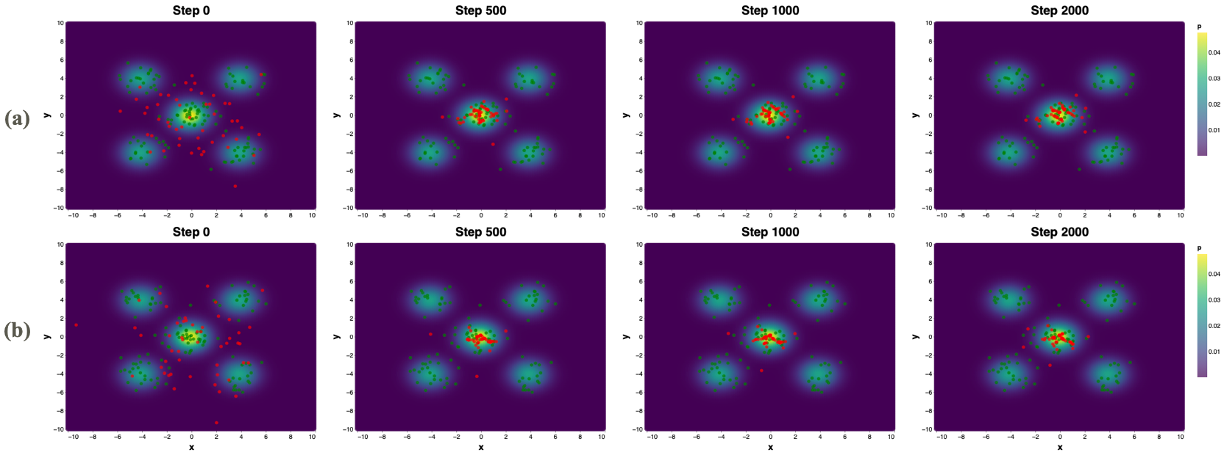}}
	\caption{The MODO problem on synthetic dataset. There are four objectives, each of which is represented by 30 particles (green points) randomly drawn from a mixture of two Gaussian distributions. The dissimilarity function $D$ is defined as the (a) KL divergence or (b) JS divergence. The objectives have a common high-density of particles. Initially 50 particles (red points) are sampled from the standard distribution to represent $q$, and then updated using MWGraD-NN. In both cases of divergences, MWGraD-NN drives the particles to the joint high density region around the origin. Note that, in this toy experiments, MWGraD-SVGD, MWGraD-Blob, MOO-SVGD cannot be used as the objective functions are not the form of energy functionals.}
	\label{fig:toy-2}
\end{figure*}

\section{Additional Experimental Results on Real-world datasets}
\label{exp:realworlddatasets}
In this section, we first examine the time complexity of MWGraD. We conduct experiments on three datasets—Multi-MNIST, Multi-Fashion, and Multi-Fashion+MNIST—to evaluate the runtime (in seconds) per training epoch for various MWGraD variants and baselines, including MOO-SVGD and MT-SGD. Table \ref{tab:runtime} presents the average runtimes of each method, computed over five runs.

We observe that MWGraD-SVGD and MWGraD-Blob have comparable runtimes to MT-SGD because they involve approximating Wasserstein gradients using kernel matrices, which is computationally efficient. In contrast, MWGraD-NN has a slower runtime due to the need to train a neural network to approximate the Wasserstein gradients of the objective functions at each iteration, which can be more time-consuming. Additionally, MOO-SVGD is slower than our methods, as it requires solving separate quadratic problems for each particle (each network). Specifically, MOO-SVGD is 9 seconds slower than MWGraD-SVGD and MWGraD-Blob, and 1 second slower than MWGraD-NN.

Second, we conduct experiments to verify the importance of key modules of MWGraD as follows. We compare the effectiveness of approximation methods, including MWGraD-SVGD, MWGraD-Blob and MWGraD-NN, for the Wasserstein gradients of objective functionals. 

Table \ref{tab:approximationmethods} (extracted from Table \ref{tab:results}) shows the ensemble accuracy numbers of compared approximation methods over three independent runs with different initializations. We observe that MWGraD-Blob achieves the best performances of 96.71\%, 97.64\% for task 1 on Multi-Fashion+MNIST and Multi-MNIST, respectively, while MWGraD-NN achieves the best performance of 97.7\% and 87.2\% on Multi-MNIST and Multi-Fashion. For task 2, MWGraD-Bob again achieves the best performance across all datasets, with average accuracies of 92.5\%, 96.7\% and 87.2\%, respectively.

Furthermore, we also explore the importance of the weight update step, and verify how the performance would change if the uniform weight were used to combine multiple Wasserstein gradients.
Tables \ref{tab:uniform-multi-fashion-mnist}, \ref{tab:uniform-multi-mnist} and \ref{tab:uniform-multi-fashion}  shown below compare the accuracies of task 1, task 2 and the average for each variant of MWGraD with and without the weight update step (we use the suffix "uniform" to indicate that the uniform weights are used). We observe that when we remove the weight update step, the performances of variants decrease. For example, for experiments on Multi-Fashion, the performance of MWGraD-Blob decreases from 87\% to 84.9\% (roughly 2\%), that of MWGraD-SVGD decreases from 84.7\% to 83.7\% (1\%) and that of MWGraD-NN decreases from 86.3\% to 85.5\% (0.7\%). These experiments demonstrate the importance of the weight update step in our proposed methods.

\begin{table}[]
    \centering
    \caption{Average runtime of compared methods on real-world datasets.}
        \begin{tabular}{l c c c}
        \hline
        \textbf{Runtime} & Multi-MNIST & Multi-Fashion & Multi-Fashion+MNIST\\
        \hline
         MOO-SVGD & 54.4 $\pm$ 0.3 & 55.1 $\pm$ 0.6 & 58.2 $\pm$ 0.3 \\
         MT-SGD & 48.1 $\pm$ 0.5 & 46.3 $\pm$ 0.3 & 44.1 $\pm$ 0.7 \\
         MWGraD-SVGD & 47.2 $\pm$ 0.9 & 45.2 $\pm$ 0.8 & 45.5 $\pm$ 0.6 \\
         MWGraD-Blob & 45.5 $\pm$ 0.6 & 46.2 $\pm$ 0.9 & 45.1 $\pm$ 0.5 \\
         MWGraD-NN & 53.3 $\pm$ 0.6 & 54.3 $\pm$ 0.5 & 56.8 $\pm$ 0.8 \\
        \hline
    \end{tabular}
    \label{tab:runtime}
\end{table}

\begin{table}[]
    \centering
    \caption{Average accuracies of approximation methods, including MWGraD-SVGD, MWGraD-Blob and MWGraD-NN, for the Wasserstein gradients of objective functionals on real-world datasets. 
.}
        \begin{tabular}{l c c c c}
        \hline
        \textbf{Datasets} & \textbf{Tasks} & MWGraD-SVGD & MWGraD-Blob & MWGraD-NN\\
        \hline
         Multi-Fashion+MNIST & 1 & 95.7$\pm$0.4 & \textbf{96.7}$\pm$0.5 & 95.9$\pm$0.4 \\
                             & 2 & 88.9$\pm$0.6 & \textbf{92.5}$\pm$0.4 & 88.2$\pm$0.3 \\
            Multi-MNIST & 1 & 94.5$\pm$0.4 & \textbf{97.6}$\pm$0.2 & \textbf{97.7}$\pm$0.5 \\
                        & 2 & 93.2$\pm$0.6 & \textbf{96.7}$\pm$0.5 & 95.5$\pm$0.4 \\
            Multi-Fashion & 1 & 85.1$\pm$0.7 & 86.8$\pm$0.3 & \textbf{87.2}$\pm$0.4 \\
                          & 2 & 84.3$\pm$0.4 & \textbf{87.2}$\pm$0.5 & 85.3$\pm$0.6 \\
        \hline
    \end{tabular}
    \label{tab:approximationmethods}
\end{table}

\begin{table}[]
    \centering
    \caption{The accuracies of task 1, task 2 and the average for each variant of MWGraD with and without the weight update step (suffix 'uniform' indicates that the uniform weights are used). Dataset: \textbf{Multi-Fashion+MNIST}.}
        \begin{tabular}{l l l l l l l}
        \hline
        \textbf{Task} & MWGraD-SVGD & MWGraD-SVGD  & MWGraD-Blob & MWGraD-Blob & MWGraD-NN & MWGraD-NN\\
        & &-uniform& & -uniform & & -uniform\\
        \hline
          1 & 95.7$\pm$0.4 & 94.1$\pm$0.9 & 96.7$\pm$0.5 & 94.2$\pm$0.6 & 95.9$\pm$0.4 & 95.5$\pm$0.3\\
         2 & 88.9$\pm$0.6 & 87.2$\pm$0.8 & 92.5$\pm$0.4 & 92.8$\pm$0.8 & 88.2$\pm$0.3 & 87.9$\pm$0.7\\
         \textbf{Avg} & 92.3 & 90.7 & 94.6 & 93.5 & 92.1 & 91.7\\
        \hline
    \end{tabular}
    \label{tab:uniform-multi-fashion-mnist}
\end{table}

\begin{table}[]
    \centering
    \caption{The accuracies of task 1, task 2 and the average for each variant of MWGraD with and without the weight update step. Dataset: \textbf{Multi-MNIST}.}
        \begin{tabular}{l l l l l l l}
        \hline
        \textbf{Task} & MWGraD-SVGD & MWGraD-SVGD  & MWGraD-Blob & MWGraD-Blob & MWGraD-NN & MWGraD-NN\\
        & &-uniform& & -uniform & & -uniform\\
        \hline
           1 & 94.5$\pm$0.4 & 94.6$\pm$0.6 & 97.6$\pm$0.2 &95.9$\pm$0.5 & 97.7$\pm$0.5 & 97.1$\pm$0.6\\
            2 & 93.2$\pm$0.6 & 93.5$\pm$0.4 & 96.7$\pm$0.7 & 97.2$\pm$0.6 & 95.5$\pm$0.4 & 93.9$\pm$0.5\\
 \textbf{Avg} & 93.9 & 94.1 & 97.2 & 96.6 & 96.6 & 95.5 \\
        \hline
    \end{tabular}
    \label{tab:uniform-multi-mnist}
\end{table}

\begin{table}[]
    \centering
    \caption{The accuracies of task 1, task 2 and the average for each variant of MWGraD with and without the weight update step. Dataset: \textbf{Multi-Fashion}.}
        \begin{tabular}{l l l l l l l}
        \hline
        \textbf{Task} & MWGraD-SVGD & MWGraD-SVGD  & MWGraD-Blob & MWGraD-Blob & MWGraD-NN & MWGraD-NN\\
        & &-uniform& & -uniform & & -uniform\\
        \hline
        1 & 85.1$\pm$0.7 & 85.2$\pm$0.4 & 86.8$\pm$0.4 & 84.1$\pm$0.7 & 87.2$\pm$0.4 & 85.9$\pm$0.3 \\
        2 & 84.3$\pm$0.4 & 82.1$\pm$0.6 & 87.2$\pm$0.5 & 85.8$\pm$0.6 & 85.3$\pm$0.6 & 85.1$\pm$0.8\\
        \textbf{Avg} & 84.7 & 83.7 & 87.0 & 84.9 & 86.3 & 85.5 \\
        \hline
    \end{tabular}
    \label{tab:uniform-multi-fashion}
\end{table}

\end{document}